\documentclass{article}

\usepackage{microtype}
\usepackage{graphicx}
\usepackage{subfigure}
\usepackage{booktabs} 

\usepackage{hyperref}

\usepackage[accepted]{icml2025}

\usepackage{amsmath}
\usepackage{amssymb}
\usepackage{mathtools}
\usepackage{amsthm}

\usepackage[capitalize,noabbrev]{cleveref}

\theoremstyle{plain}
\newtheorem{theorem}{Theorem}[section]
\newtheorem{proposition}[theorem]{Proposition}
\newtheorem{lemma}[theorem]{Lemma}
\newtheorem{corollary}[theorem]{Corollary}
\theoremstyle{definition}

\newtheorem{assumption}[theorem]{Assumption}
\theoremstyle{remark}
\newtheorem{remark}[theorem]{Remark}

\usepackage[textsize=tiny]{todonotes}

\icmltitlerunning{Robust Bayesian Ambiguity Sets}

\RequirePackage{amsfonts}
\RequirePackage{amsmath}
\RequirePackage{amssymb}
\RequirePackage{amsthm}

\usepackage[T3,T1]{fontenc}
\DeclareSymbolFont{tipa}{T3}{cmr}{m}{n}
\DeclareMathAccent{\post}{\mathalpha}{tipa}{16}

\DeclareMathOperator*{\argmin}{arg\,min}

\DeclareMathOperator*{\DP}{DP}

\DeclareMathOperator{\simiid}{\overset{iid}{\sim}}
\DeclareMathOperator*{\st}{s.t.}
\DeclareMathOperator*{\Dir}{Dir}

\DeclareMathOperator*{\ri}{ri}

\newcommand{\cA}{{\cal A}}
\newcommand{\cB}{{\cal B}}
\newcommand{\cC}{{\cal C}}

\newcommand{\cH}{{\cal H}}
\newcommand{\cI}{{\cal I}}

\newcommand{\cK}{{\cal K}}

\newcommand{\cN}{{\cal N}}

\newcommand{\cP}{{\cal P}}

\newcommand{\cX}{{\cal X}}

\newcommand{\bE}{{\mathbb E}}
\newcommand{\bF}{{\mathbb F}}

\newcommand{\bP}{{\mathbb P}}
\newcommand{\bQ}{{\mathbb Q}}
\newcommand{\bR}{{\mathbb R}}

\newcommand{\NPL}{\Pi_\text{NPL}}

\newcommand{\MMD}{\mathbb{D}_k}
\newcommand{\KL}{d_{\text{KL}}}

\newcommand{\truedist}{\bP^\star}
\usepackage{xspace}
\newcommand{\bdro}{BDRO\xspace}
\newcommand{\drobas}{DRO-BAS\xspace}
\newcommand{\drobaspe}{DRO-BAS\textsubscript{PE}\xspace}
\newcommand{\drobaspp}{DRO-BAS\textsubscript{PP}\xspace}
\newcommand{\drorobas}{DRO-RoBAS\xspace}
\newcommand{\bas}{BAS\xspace}
\newcommand{\robas}{RoBAS\xspace}
\newcommand{\baspe}{BAS\textsubscript{PE}\xspace}
\newcommand{\baspp}{BAS\textsubscript{PP}\xspace}

\newcommand{\thm}{Theorem\xspace}

\newcommand{\lem}{Lemma\xspace}

\newcommand{\fig}{Fig.\xspace}

\newcommand{\pred}{\bP_n^{\text{pred}}}
\newcommand{\prednpl}{\bP_n^{\text{pred(NPL)}}}
\newcommand{\prednplapprox}{\hat{\bP}_n^{\text{pred(NPL)}}}
\newcommand{\dphatxi}{\widehat{\DP}_{\xi_{1:n}}}
\newcommand{\dpxi}{\DP_{\xi_{1:n}}}
\newcommand{\thetamapk}{\theta_k}

\makeatletter
\newcommand{\oset}[2]{{\mathpalette\o@set{{#1}{#2}}}}
\newcommand{\o@set}[2]{\o@@set{#1}#2}
\newcommand{\o@@set}[3]{%
  \vbox{\offinterlineskip
    \ialign{\hfil##\hfil\cr
      $\m@th\o@set@demote{#1}#2$\cr
      \noalign{\vskip0.2pt}
      $\m@th#1#3$\cr
    }%
  }%
}
\newcommand{\o@set@demote}[1]{%
  \ifx#1\displaystyle\scriptstyle\else
  \ifx#1\textstyle\scriptstyle\else
  \scriptscriptstyle\fi\fi
}
\makeatother

\newenvironment{talign*}
{\let\displaystyle\textstyle\csname align*\endcsname}
{\endalign}

\definecolor{baspe}{HTML}{984ea3}
\definecolor{baspp}{HTML}{4daf4a}
\definecolor{bdro}{HTML}{ff7f00}
\definecolor{robas}{HTML}{377eb8}

\begin{document}

\twocolumn[
\icmltitle{Decision Making under Model Misspecification:\\ DRO with Robust Bayesian Ambiguity Sets}



\icmlsetsymbol{equal}{*}

\begin{icmlauthorlist}
\icmlauthor{Charita Dellaporta}{equal,xxx,zzz}
\icmlauthor{Patrick O'Hara}{equal,yyy}
\icmlauthor{Theodoros Damoulas}{yyy,zzz}
\end{icmlauthorlist}

\icmlaffiliation{yyy}{Department of Computer Science, University of Warwick, Coventry, UK}
\icmlaffiliation{xxx}{Department of Statistics, University College London, London, UK}
\icmlaffiliation{zzz}{Department of Statistics, University of Warwick, Coventry, UK}

\icmlcorrespondingauthor{Charita Dellaporta}{h.dellaporta@ucl.ac.uk}
\icmlcorrespondingauthor{Patrick O'Hara}{Patrick.H.O-Hara@warwick.ac.uk}


\vskip 0.3in
]



\renewcommand{\isaccepted}{}

\printAffiliationsAndNotice{\icmlEqualContribution} 

\begin{abstract}
Distributionally Robust Optimisation (DRO) protects risk-averse decision-makers by considering the worst-case risk within an ambiguity set of distributions based on the empirical distribution or a model. To further guard against finite, noisy data, model-based approaches admit Bayesian formulations that propagate uncertainty from the posterior to the decision-making problem. However, when the model is misspecified,
the decision maker must stretch the ambiguity set to contain the data-generating process (DGP), leading to overly conservative decisions. We address this challenge by introducing DRO with Robust, to model-misspecification, Bayesian Ambiguity Sets (\drorobas). These are Maximum Mean Discrepancy ambiguity sets centred at a robust posterior predictive distribution that incorporates beliefs about the DGP. We show that the resulting optimisation problem obtains a dual formulation in the Reproducing Kernel Hilbert Space and we give probabilistic guarantees on the tolerance level of the ambiguity set. Our method outperforms other Bayesian and empirical DRO approaches in out-of-sample performance on the Newsvendor and Portfolio problems with various cases of model misspecification.
\end{abstract}

\section{Introduction}
\label{sec:intro}

Decision-makers frequently encounter the challenge of optimising under uncertainty since the data-generating process (DGP) is not fully known. As a result, they must rely on available data and model families to estimate the DGP within the optimisation objective. However, the data may be noisy, the data distribution might change with time, or the model might be misspecified or poorly fitted, leading to \textit{distributional uncertainty}. Decision-making in this setting is a critical challenge in various applications such as inventory planning \citep{black2022parametric}, portfolio optimisation \citep{li2013portfolio} and distribution shifts in machine learning applications \citep{zhang2020coping}. 

\begin{figure}[t]
\begin{center}
\centerline{\includegraphics[width=0.9\columnwidth]{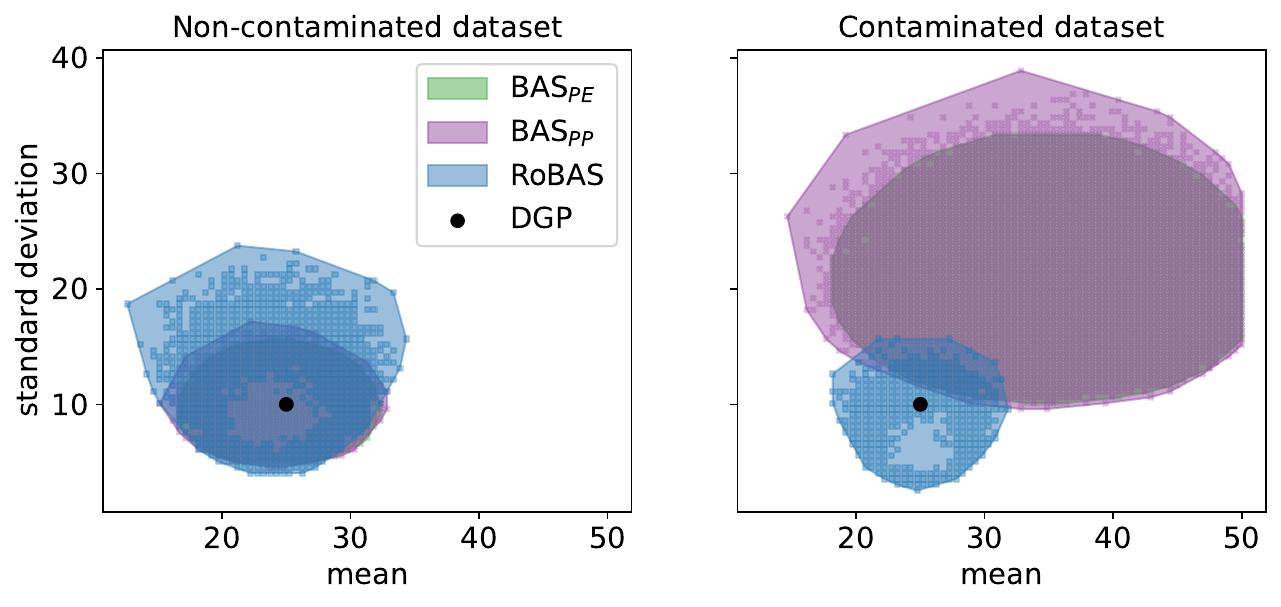}}
\caption{Illustration of (approximated) \baspe, \baspp and \robas (ours) with $\bP_\theta = \cN(\mu, \sigma^2)$ over a grid of $(\mu, \sigma)$ pairs for a fixed $\epsilon$. In the well-specified case (left), all ambiguity sets include the DGP while \robas covers a slightly bigger area than \baspe and \baspp. For a contaminated dataset (right) \robas continues to contain the DGP and maintains a similar area, whereas the \bas formulations exclude it and cover a much larger area of distributions further away from the DGP.}
\label{fig:bas-vs-robas}
\end{center}
\vskip -0.36in
\end{figure}

A risk-averse decision-maker might choose to hedge against distributional uncertainty by finding the decision that minimises the \textit{worst-case} risk over a set of distributions. This worst-case protection is at the heart of \textit{Distributionally Robust Optimisation (DRO)} that defines an ambiguity set of distributions with respect to an estimator of the DGP. This can be fully data-driven, using the empirical measure of the observations \citep[see e.g.][]{Kuhn2019, staib2019distributionally, zhu2021kernel, hu2013kullback}, or model-based when expert knowledge is available to fit a model to available data, resulting in a model-based estimator for the DGP \citep[][]{iyengar2023hedging, michel2021modeling, michel2022distributionally}. Both approaches are sensitive to the choice of DGP estimator and additional estimation error might exist in model-based DRO due to a poor fit or model uncertainty.

To overcome this, recently developed Bayesian formulations of DRO use posterior beliefs to inform the optimisation problem \citep{shapiro2023bayesian} or the ambiguity set itself \citep{dellaporta2024decision}. However, these methods inherit the sensitivity of Bayesian posteriors to model misspecification \citep[see e.g.][]{grunwald2012safe, walker2013bayesian}. 
A key goal in DRO methodology is to choose the size of the ambiguity set such that the DGP falls within it with high probability, as illustrated in Fig. \ref{fig:bas-vs-robas} (left) for the two formulations of Bayesian Ambiguity Sets (\bas) \citep{dellaporta2024decision} and our proposed Robust \bas (\robas). 
If the estimate is not accurate - for example, when the model is misspecified - then a much larger size will be required to contain the DGP as illustrated in Fig. \ref{fig:bas-vs-robas} (right).
The price to pay for this large size is the inclusion of many probability distributions that are unlikely to occur, and which could be very pessimistic with respect to the objective function, leading to an overly conservative decision.
If the decision maker wrongly assumes the model is well-specified and incorrectly chooses an overly optimistic ambiguity set size, then the DGP may not lie in the set, and the decision could be overly optimistic compared to the out-of-sample outcome, often referred to as the optimiser's curse  \citep{Kuhn2019}. 

Uncertainty over the DGP is extensively studied outside stochastic optimisation. Consider a parametric model $\bP_\theta$, indexed by the parameter of interest $\theta$. In the Bayesian framework, uncertainty about the parameter is typically expressed directly through prior beliefs. However, recent work in robust Bayesian inference by \citet{lyddon2018nonparametric} takes a different approach: uncertainty about the parameter is now induced by uncertainty, in the form of prior beliefs, in the DGP. This concept lies at the core of the Bayesian Nonparametric Learning (NPL) framework \citep{lyddon2018nonparametric, fong2019scalable} which relaxes the well-specified model assumption imposed by standard Bayesian inference. 
In this spirit, we approach the challenge of DRO under model misspecification by extending the recently proposed \drobas  
\citep{dellaporta2024decision} framework to tackle model misspecification through a robust NPL posterior coupled with the Maximum Mean Discrepancy (MMD) inside the ambiguity set; thus introducing DRO with Robust Bayesian Ambiguity Sets (\drorobas). While \drobas targets distributional uncertainty with respect to the DGP, the robustness offered is not sufficient under model misspecification, see \fig \ref{fig:bas-vs-robas}.

\section{Background} \label{sec:background}
Let $x \in \cX \subseteq \bR^d$ be a decision-making variable for the cost function $f: \cX \times \Xi \rightarrow \bR$ with data space $\Xi \subseteq \bR^D$. Let $\{\xi_i\}_{i=1}^n \simiid \bP^\star$ be observations from the DGP $\bP^\star \in \cP(\Xi)$, where $\cP(\Xi)$ denotes the space of Borel probability measures on $\Xi$. Furthermore, consider a parametric model family $\cP_\Theta := \{\bP_\theta: \theta \in \Theta\} \subset \cP(\Xi)$ indexed by parameter of interest $\theta \in \Theta \subseteq \bR^s$. We say the model is \emph{misspecified} if $\bP^\star \notin \cP_\Theta$. DRO methods construct an ambiguity set $\cA$, based on an estimator of $\bP^\star$, called the \emph{nominal distribution}, and minimise the worst-case expected cost over $\cA$.

The DRO literature typically categorises ambiguity sets into two classes \citep{rahimian2019distributionally}: {\em moment-based} and {\em discrepancy-based}.
The former contain distributions that satisfy constraints on the moments of $\bP^\star$, without necessarily considering an estimator.
In contrast, discrepancy-based ambiguity sets consist of distributions close to the nominal according to a specified discrepancy measure. Examples include Integral Probability Metrics (IPMs) \citep{husain2022distributionally}, such as the Wasserstein distance \citep{Kuhn2019} and the MMD \citep{staib2019distributionally}, as well as 
$\phi$-divergences like the Kullback-Leibler (KL) divergence \citep{hu2013kullback}.
Regardless of the choice of ambiguity set, the resulting minimax problem can be seen as a game between the decision maker who chooses $x$ and an adversary who chooses the worst-case distribution in $\cA$:
\begin{align} \label{opt:dro}
    \min_{x \in \cX}~{\sup_{\bP \in \cA}~{\bE_{\xi \sim \bP}[f_x(\xi)]}}
\end{align}
where $f_x(\xi) := f(x,\xi)$. Although most discrepancy-based DRO methods are fully empirical, i.e. the estimator is obtained only via $\xi_{1:n}$, sometimes, such as in regression settings, the decision-maker needs to model the variables' relationship via a model family $\cP_\Theta$, describing the DGP. Model-based DRO methods \citep[e.g.][]{iyengar2023hedging, michel2021modeling, michel2022distributionally} use the observations to obtain an estimator $\bP_{\hat{\theta}} \in \cP_\Theta$ and use this as the nominal distribution. Thus, a poorly chosen $\bP_{\hat{\theta}}$ far from $\bP^\star$ (in some distance sense) requires a large $\cA$, leading to overly pessimistic decisions. This has led to Bayesian formulations of DRO which propagate uncertainty about $\theta$ in the optimisation problem.

\subsection{
Bayesian Formulations of DRO}
\citet{shapiro2023bayesian} introduced \textbf{Bayesian DRO (\textcolor{bdro}{\bdro})}  which defines an \textit{expected worst-case} risk objective: 
\begin{align} \label{eq:bdro}
    \min_{x \in \cX}~{\color{bdro}\bE_{ \Pi(\theta \mid \xi_{1:n})}}~{\left[\sup_{\displaystyle\color{bdro}{\bP:  \KL(\bP || \bP_{\theta}) \leq \epsilon}}~{\bE_{\xi \sim \bP}\left[f_x(\xi)\right]}\right]}
\end{align}
where $\Pi(\theta \mid \xi_{1:n})$ denotes the parameter posterior distribution for model family $\cP_\Theta$. However, risk-averse decision-makers are interested in worst-case risk formulations. For this reason, \citet{ dellaporta2024decision} proposed two formulations of the {\textbf{DRO with Bayesian Ambiguity Sets (\drobas)}} that correspond to a \textit{worst-case} optimisation problem with ambiguity sets informed by the standard Bayesian posterior. In particular, they defined \textbf{\textcolor{baspp}{\drobaspp}}:
\begin{align} 
  \min_{x \in \cX}~\sup_{\color{baspp}\displaystyle\bP: \KL\left(\bP || \bE_{\Pi(\theta \mid \xi_{1:n})}[\bP_\theta]\right) \leq \epsilon}~{\bE_{\xi\sim\bP}\left[ f_x(\xi) \right]},
  \label{eq:dro-bas-pp} 
\end{align}
based on a KL-based ambiguity set with nominal distribution the posterior predictive, and \textbf{\textcolor{baspe}{\drobaspe}}:
\begin{align}
  \min_{x \in \cX}~\sup_{\color{baspe}\displaystyle\bP: \bE_{\Pi(\theta \mid \xi_{1:n})}[\KL(\bP || \bP_\theta)] \leq \epsilon}~{\bE_{\xi\sim\bP}\left[ f_x(\xi) \right]}. 
  \label{eq:dro-bas-pe} 
\end{align}
which considers the expected KL under the posterior distribution. 
The authors showcased improved out-of-sample robustness compared to \bdro in a number of exponential family models. Although \drobas in the standard Bayesian setting offers an intuitive, posterior-informed ambiguity set, it can be severely affected by model misspecification. Indeed, \baspe only considers probability measures $\bP$ that are absolutely continuous with respect to $\bP_\theta$ (denoted by $\bP \ll \bP_\theta$) and also admit an expected KL divergence close enough to $\bP_\theta$. 
Since the expectation is informed by the posterior, a non-robust posterior will likely lie far away from the DGP. 
Similarly, \baspp considers probability measures $\bP$ that admit small KL-divergence with respect to $\pred$, where $\pred := \bE_{\theta \sim \Pi(\theta \mid \xi_{1:n})}[\bP_\theta]$ and $\bP \ll \pred$. Hence, the sensitivity of the Bayesian posterior will propagate to the posterior predictive and the resulting ambiguity set. A similar observation was made by \citet{shapiro2023bayesian} for the \bdro method in the misspecified case. 

To remedy this, we exploit the flexibility of the \drobas framework which allows us to choose a different posterior distribution and discrepancy measure, suitable for model misspecification. The notion of targeting a different discrepancy measure, other than the KL divergence, to induce robustness in the Bayesian posterior has been well established in the Bayesian inference literature. The NPL posterior \citep{lyddon2018nonparametric, fong2019scalable} was introduced to, among others, remedy the sensitivity of Bayesian inference to model misspecification by removing the assumption that the model is correct. This is done by setting uncertainty, via nonparametric prior beliefs, directly on the DGP rather than on the parameter of interest. 
Incorporating DGP uncertainty in decision-making has also been considered by \citet{wang2023learning} who explored a nonparametric Dirichlet Process (DP) model for the DGP. Unlike the current paper, this work is not suited for parametric models and considers a weighted objective, with only one counterpart corresponding to a worst-case risk. We focus on decision-making under \emph{parametric} models, which are especially useful for interpretability in decision-making, while also incorporating \emph{nonparametric} prior beliefs about the DGP.

To achieve this, we leverage the work of \citet{dellaporta2022robust} who extended the NPL posterior to discrepancy-based loss functions and showed robustness guarantees when the Maximum Mean Discrepancy (MMD) is used. Using an MMD-based loss allows us to also employ the MMD to construct the ambiguity set. NPL is a natural choice for the \drobas framework as distributional uncertainty in DRO stems directly from uncertainty in the DGP. 

\subsection{Robust NPL Posterior}
In this work, we propose an alternative formulation of \drobas, based on the robust NPL posterior introduced by \citet{lyddon2018nonparametric, fong2019scalable}. We introduce a DP prior~$\bQ \sim \DP(\alpha, \bF)$ on the DGP $\bP^\star$ where $\alpha > 0$ and $\bF \in \cP(\Xi)$. Here, $\bF$ represents our prior beliefs about the DGP and the concentration parameter $\alpha$ dictates the strength of the beliefs with $\alpha = 0$ representing a non-informative prior. To see this, note that given data $\xi_{1:n}$, the posterior is 
\begin{align} \label{eq:dp-posterior}
\begin{split}
    \bQ~|~\xi_{1:n} &\sim \DP(\alpha^\prime, \bF^\prime), \\
    \alpha^\prime := \alpha + n, \quad \bF^\prime &:= \frac{\alpha}{\alpha + n} \bF + \frac{n}{\alpha + n} \bP_n
\end{split}
\end{align}
where $\bP_n := \frac{1}{n}~\sum_{i=1}^n~{\delta_{\xi_i}}$ is the empirical measure and $\delta_{\xi}$ denotes the Dirac measure at $\xi \in \Xi$. For $\alpha = 0$, the DP posterior is centred directly on $\bP_n$. If $\bP^\star$ was known, we could directly compute:
\begin{align} \label{eq:npl-target}
    \theta_L (\bP^{\star}) := \argmin_{\theta \in \Theta} \bE_{\xi \sim \bP^{\star}} [L(\xi; \theta)]
\end{align}
where $L:\Xi \times \Theta \rightarrow \bR$ denotes any loss function. Note that this NPL objective does not assume that the model is well-specified but simply looks for the most likely value of $\theta$ under the expectation of the DGP or, equivalently, the parameter value that best describes the data under the candidate model. Since $\bP^\star$ is unknown and we instead have a nonparametric posterior over it, we can propagate our posterior beliefs to the parameter of interest through the push-forward measure $(\theta_{L})_{\#}(\DP(\alpha', \bF'))$ to give a posterior $\NPL$ on $\Theta$. Sampling from this posterior can be done through the Posterior Bootstrap \citep{fong2019scalable}: For B posterior bootstrap iterations, at iteration~$j \in [B]$:
\begin{enumerate}
    \item Sample $\bQ^{(j)}$ from the posterior $\DP(\alpha^\prime, \bF^\prime)$.
    \item Compute $\theta^{(j)} = \theta_L(\bQ^{(j)})$ where $\theta_L(\cdot)$ as in (\ref{eq:npl-target}). 
\end{enumerate}
\citet{dellaporta2022robust} suggested using a discrepancy-based loss function in (\ref{eq:npl-target}) which we introduce below. 

\subsection{Maximum Mean Discrepancy} \label{sec:mmd}
The MMD belongs to the family of IPMs \citep{muller1997integral}. 
Let $\mathcal{H}_k$ be a Reproducing Kernel Hilbert Space (RKHS), for kernel $k: \Xi \times \Xi \rightarrow \bR$ and norm $\| \cdot \|_{k}$. For $\cP_k(\Xi) := \{\bP \in \cP(\Xi): \int_{\Xi} \sqrt{k(\xi, \xi)} \bP(d\xi) < \infty \}$, the MMD between $\bP, \bQ \in \cP_k(\Xi)$ is defined as:
\begin{align} \label{eq:mmd-sup}
    &\MMD(\bP, \bQ) := \sup_{f \in \mathcal{H}_k, \| f \|_{k} \leq 1}\left| \bE_{\bP}[f(\xi)] - \bE_{\bQ}[f(\xi)] \right|.
\end{align}
\citet{dellaporta2022robust} define the NPL target in \eqref{eq:npl-target} as:
\begin{align} \label{eq:npl-target-mmd}
    \theta_k (\bP^{\star}) := \argmin_{\theta \in \bR^s} \MMD(\bP^{\star}, \bP_{\theta}).
\end{align}
One of the attractive properties of the MMD is that the supremum in (\ref{eq:mmd-sup}) can be obtained in closed form as
\begin{equation}
\begin{split}
    \MMD^2(\bP, \bQ) &= \bE_{\xi, \xi' \sim \bP} [k(\xi, \xi')]  
 - 2 \bE_{\xi \sim \bP, \xi' \sim  \bQ} [k(\xi, \xi')]  \\
&\quad + \bE_{\xi, \xi' \sim \bQ} [k(\xi, \xi')] 
\end{split}
\end{equation}
and can be approximated via sampling \citep[see e.g.][]{gretton2012kernel}.
The resulting NPL posterior with the MMD is called an NPL-MMD posterior.

We explore the gains of using the MMD both in the robust NPL posterior and in the Bayesian ambiguity set. Since the NPL-MMD posterior will target the point in the model family closest (w.r.t. the MMD) to $\bP^\star$, the MMD-based ambiguity set may require a smaller radius to include $\bP^\star$, resulting in less conservative decisions.
The MMD has previously been used as a distance metric in the DRO context by \citet{staib2019distributionally} who considered the ambiguity set $\cB^k_{\epsilon}(\bP_n):= \{\bP \in \cP_k(\Xi): \MMD(\bQ, \bP_n) \leq \epsilon\}$ where $\epsilon >0$ is the radius of the MMD ball, 
$k$ is the kernel and $\bP_n:= \frac{1}{n} \sum_{i=1}^n \delta_{\xi_i}$ is the empirical measure of the observations $\xi_{1:n}$. The same ambiguity set forms a special case of Kernel DRO which was introduced by \citet{zhu2021kernel}. \citet{chen2022distributionally} generalised this to conditional kernel DRO, leveraging conditional distributions and \citet{romao2023distributionally} explored this framework for dynamic programming. 

\section{DRO With Robust Bayesian Ambiguity Sets} \label{sec:dro-bas-misspec}
We propose a robust version of \drobaspp (\ref{eq:dro-bas-pp}) via the MMD and the NPL-MMD posterior predictive defined as:
\begin{align} \label{eq:pp-npl}
    \prednpl := \bE_{\bQ \sim \DP(c^\prime, \bF^\prime)}[\bP_{\theta_k(\bQ)}].
\end{align}
Notice that contrary to the standard Bayesian posterior predictive, $\prednpl$ is defined through marginalisation over the nonparametric posterior over the DGP.
Since the MMD can be approximated only via samples (Section \ref{sec:mmd}), a closed-form density for the predictive in (\ref{eq:pp-npl}) is not required. 
We define the following \textbf{Robust Bayesian Ambiguity Set (\textcolor{robas}{\robas})} with the NPL Posterior Predictive:
\begin{align*} 
    \cB^k_{\epsilon}(\prednpl) := \{\bP \in \cP_k(\Xi): \MMD(\bP, \prednpl) \leq \epsilon\}. 
\end{align*}
Note that $\cB^k_{\epsilon}(\prednpl)$ forms a ball when $k$ is a characteristic kernel, as this makes the MMD a probability metric. This property is desirable as it guarantees that the MMD will be zero if and only if $\bP \equiv \prednpl$.
We hence obtain the following {\textbf{\color{robas} \drorobas}} worst-case risk problem: 
\begin{align}
    \min_{x \in \cX} \sup_{\displaystyle\color{robas}\bP \in \cB_\epsilon^k(\prednpl)} \bE_{\xi \sim \bP}[f_x(\xi)] \label{eq:dro-robas-pp}
\end{align}
Similarly to \drobas, this optimisation problem corresponds to a worst-case risk over a set of probability measures informed by posterior beliefs. Posterior beliefs about $\theta$ are obtained \textit{via posterior beliefs about the DGP} and the map $\theta_k$ in (\ref{eq:npl-target-mmd}). Since the goal of DRO is to target uncertainty about the DGP, the NPL posterior is a natural choice to inform the ambiguity set as it takes into account any prior beliefs about the DGP. Moreover, by targeting the MMD, rather than the KL divergence as in the \bas case, \robas is \emph{not} restricted to probability measures that are absolutely continuous with respect to $\bP_\theta$. 

Intuitively, \robas is expected to be a better-informed ambiguity set than \bas when the model is misspecified since it is informed by a robust posterior predictive and a robust discrepancy measure.
This is better understood through a toy example.
\fig~\ref{fig:outliers} shows a Gaussian location model in the presence of outliers.
The top panel shows the DGP of the training data contaminated with 20\% of outliers along with a pathological model $\bP_{\text{pathological}}$ case with a mean larger than that of the DGP.
In the \baspe case, the expected KL from the DGP to the model is significantly larger than that from the pathological model $\bP_{\text{pathological}}$ due to the sensitivity of the Bayesian posterior to outliers. A similar result holds for the KL divergence between the posterior predictive and the DGP and pathological model in the \baspp case. In contrast, the MMD from the NPL posterior predictive to the DGP is much smaller compared to that of the pathological model. Alternative \drorobas formulations
are provided in Appendix~\ref{sec:alt}.
Proofs of this section are in Appendix~\ref{app:proofs}. 

\begin{figure}[t!]
\vskip 0.2in
    \centering \includegraphics[width=0.75\linewidth]{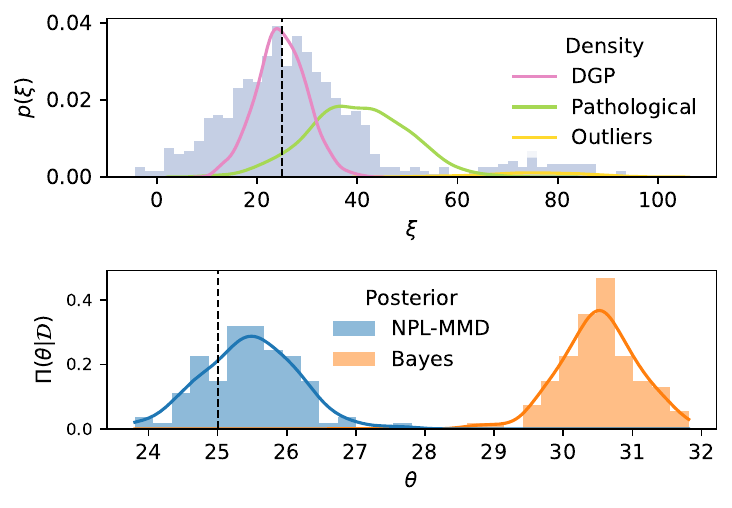}
    \caption{Contaminated Gaussian location example. Top: Histogram of observed data along with the true (DGP), outlier and pathological densities. Bottom: Posterior marginal distributions for NPL-MMD and standard Bayes. The true mean is indicated with a dotted line. For \baspe it holds that $\bE_{\Pi_{\text{Bayes}}}[\KL(\bP_{\text{pathological}}, \bP_\theta)] \approx 0.17<0.42 \approx \bE_{\Pi_{\text{Bayes}}}[\KL(\truedist, \bP_\theta)]$ and similarly for \baspp it holds that $\KL(\bP_{\text{pathological}}, \pred) \approx 0.18 < 0.38 \approx \KL(\truedist, \pred)$. In contrast for \robas we have $\MMD(\bP_{\text{pathological}}, \prednpl) \approx 0.65 > 0.55 \approx \MMD(\truedist, \prednpl)$. This example is inspired by \fig 1 of \citet{gao2023distributionally}.}
    \label{fig:outliers}
    \vskip -0.2in
\end{figure}

\subsection{Duality of the \drorobas Problem} \label{sec:duality}
We first formulate our optimisation problem as a kernel DRO problem \citep{zhu2021kernel}. This allows us to obtain a dual formulation of (\ref{eq:dro-robas-pp}) in the RKHS which can be optimised using kernel methods. Let $\phi: \Xi \rightarrow \cH_k$ denote the feature map associated with kernel $k$ and let $\mu_{\bP} \in \cH_k$ denote the kernel mean embedding of the probability measure $\bP \in \cP_k(\Xi)$, i.e. $\mu_{\bP} := \bE_{\xi \sim \bP}\left[\phi(\xi) \right]$. Then the MMD is equivalently defined as $\MMD(\bP, \bQ) = \| \mu_{\bP} - \mu_{\bQ} \|_k$ \citep[][\lem 4]{gretton2012kernel}. Throughout we assume that $\prednpl \in \cP_k(\Xi)$. Note that this condition is trivially satisfied for a bounded kernel $k$. Consider the following set
satisfying the conditions of \robas: 
\begin{align} \label{eq:C}
    \cC^\star := \{\mu \in \cH_k:  \| \mu - \mu_{\prednpl} \|_{k} \leq \epsilon\}.
\end{align}
The associated ambiguity set induced by $\cC^\star$ is 
\begin{align} \label{eq:Kc}
\begin{split}
    \cK_{\cC^\star} &:= \{\bP \in \cP_k(\Xi), \mu_\bP  \in \cC^\star\} 
    \equiv \cB_\epsilon^k(\prednpl)
\end{split}
\end{align}
and our \drorobas problem in (\ref{eq:dro-robas-pp}) is equivalent to:
\begin{gather} \label{eq:kernel-dro-primal}
\begin{aligned}
    \min_{x \in \cX}~\sup_{\bP, \mu_\bP}~\bE_{\xi \sim \bP}[f_x(\xi)] \hspace{0.2cm}\text{s.t.} \hspace{0.2cm} \bP \in \cP_k(\Xi),\hspace{0.2cm} \mu_\bP \in \cC^\star. 
\end{aligned}
\end{gather}
The equivalence can be seen through the $\cC^\star$-induced ambiguity set for distributions $\bP$ which can be written as $\cK_{\cC^\star}$.

Before we proceed to our dual formulation, we introduce the main result from \citet{zhu2021kernel} which gives a dual formulation of Kernel DRO problems for general sets~$\cC$ satisfying certain assumptions. 
\begin{theorem}[\citet{zhu2021kernel}, \thm 3.1] \label{thm:kernel-dro}
    Assume $\cC \subset \cH_k$ is closed convex, $f_x(\cdot)$ is proper, upper semi-continuous, and $\ri(\cK_{\cC}) \neq \emptyset$, where $\ri(\cK_{\cC})$ denotes the relative interior of $\cK_{\cC}$. Then the primal problem:  
    \begin{align*}
    &\min_{x}~\sup_{\bP, \mu}~\bE_{\xi \sim \bP}[f_x(\xi)] \hspace{0.3cm} \st \hspace{0.2cm} \bP \in \cP,~\mu_\bP = \mu,~\mu \in \cC
    \end{align*}
    is equivalent to:
    \begin{align*}
        \min_{x, g_0 \in \bR, g \in \cH_k} g_0 + \delta_\cC^\star(g) \hspace{0.1cm}
        \st \hspace{0.1cm} f_x(\xi) \leq g_0 + g(\xi), \forall \xi \in \Xi
    \end{align*}
    where $\delta_\cC^\star(g) := \sup_{\mu \in \cC} \left<g, \mu \right>_{\cH_k}$ the support function of $\cC$.
\end{theorem}

This Theorem gives an effective way to transition from the primal to the dual formulation by using the support function of the set $\cC$. Importantly, in contrast to other kernel-based dual formulations \citep[e.g.][]{staib2019distributionally}, this theorem does \emph{not} require the objective function $f$ to be a member of the RKHS $\cH_k$. 
We first derive the support function of $\cC^\star$ in the \drorobas case. We denote $\bE_{\bQ \sim \DP(\alpha', \bF')}$ and $\bE_{\xi \sim \bP_{\theta_k(\bQ)}}$ by $\bE_{\dpxi}$ and $\bE_{\bP_{\theta_k(\bQ)}}$ respectively. 
\begin{proposition} \label{prop:support-fn}
    Let $\cC^\star$ be defined by (\ref{eq:C}). Then we have $\delta^\star_{\cC^\star}(g) = \bE_{\DP_{\xi_{1:n}}} \left[ \bE_{\bP_{\theta_k(\bQ)}} \left[ g(\xi) \right] \right] + \epsilon \| g \|_k$.
\end{proposition}
We can now apply \thm \ref{thm:kernel-dro} to our problem. 
\begin{corollary} \label{cor:bas-dro-mmd-dual}
    Let $\cC^\star$ as in (\ref{eq:C}) and $f_x(\cdot)$ proper, upper semi-continuous. Then problem (\ref{eq:kernel-dro-primal}) is equivalent to:
    \begin{gather}  \label{eq:kernel-dro-dual-mmd}
    \begin{aligned}
        &\min_{x, g_0 \in \bR, g \in \cH_k} && g_0 +  \bE_{\DP_{\xi_{1:n}}} \left[ \bE_{ \bP_{\theta_k(\bQ)}} \left[ g(\xi) \right] \right] + \epsilon \| g \|_{k} \\
        &\text{\normalfont{subject to}} && f_x(\xi) \leq g_0 + g(\xi), \hspace{2em}\forall \xi \in \Xi.
    \end{aligned}
    \end{gather}
\end{corollary}

\paragraph{Computation of (\ref{eq:kernel-dro-dual-mmd}):} \label{para:discretis}
The problem in (\ref{eq:kernel-dro-dual-mmd}) can be solved by the batch approach with discretization of a semi-infinite programme (SIP) \citep{vazquez2008generalized} suggested in \citet{zhu2021kernel}, in addition to a Sample Average Approximation (SAA). 
Let $\{\hat{\xi}_i\}_{i=1}^N$ be samples from the nested expectation in (\ref{eq:kernel-dro-dual-mmd})
and $\{\zeta_j\}_{j=1}^m \subseteq \Xi^m$ be a set of discretisation points. Then the problem can be approximated by: 
\begin{gather}  \label{eq:kernel-dro-dual-mmd-disc}
\begin{aligned}
        &\min_{x, g_0 \in \bR, g \in \cH_k} && g_0 + \frac{1}{N}\sum_{i=1}^N  g(\hat{\xi}_i) + \epsilon \| g \|_{k} \\
        &\quad \quad \text{s.t.} && f_x(\zeta_j) \leq g_0 + g(\zeta_j),\hspace{1em} \forall j \in [m].
\end{aligned}
\end{gather}
We can now apply the distributional robust version of the Representer theorem \citep[][\lem B.1]{zhu2021kernel} which states that it's sufficient to parametrise $g$ by $g(\cdot) = \sum_{i=1}^N \alpha_i k(\hat{\xi}_i, \cdot) + \sum_{j=1}^m \alpha_{N + j} k(\zeta_j, \cdot)$ for some $\alpha_k \in \bR$, for all $k = 1, \ldots, N+m$.

\subsection{Tolerance Level Guarantees} \label{sec:tol}
We start by using the generalisation error results for the NLP-MMD posterior to obtain a result in probability that the DGP lies within our ambiguity set.
First, we give a concentration type bound for $\bE_{\bQ \sim \DP(\alpha', \bF')}[\MMD(\bP^\star, \bQ)]$. In practice, exact sampling from a DP is not possible, so we consider the approximation of the DP suggested in the NPL literature \citep{lyddon2018nonparametric, fong2019scalable, dellaporta2022robust} to sample during the MMD Posterior Bootstrap. In particular, denote by $\dphatxi$ the probability measure on $\cP(\Xi)$ induced by the following sampling process for $(w_{1:n}, \tilde{w}_{1:\tau}) \sim \Dir\left(1, \dots, 1, \frac{\alpha}{\tau}, \dots, \frac{\alpha}{\tau}\right)$ and $\tilde{\xi}_{1:\tau} \simiid \bF$:
\begin{align} \label{eq:approx-dp}
\begin{split}
    \bQ := \sum_{i=1}^n w_i \delta_{\xi_i} + \sum_{k=1}^\tau \tilde{w}_k \delta_{\tilde{\xi}_k} \sim \dphatxi. 
    \end{split}
\end{align}
The associated approximate posterior predictive is
$
    \prednplapprox :=  \bE_{\bQ \sim \dphatxi}[\bP_{\theta_k(\bQ)}].
$
We provide a concentration result for the MMD between
$\bP^\star$ and $\prednplapprox$ 
as the approximated predictive is used in practice.
However, similar results can be derived for the exact case with $\prednpl$. Additionally, all theoretical results regarding the duality of the \drorobas framework from Section \ref{sec:duality} hold exactly the same for the approximated DP as they are proven for a general posterior. 
We make the following assumptions:
\begin{assumption} \label{ass:minimiser}
    For every $\bQ \in \cP(\Xi)$ there exists $c > 0$ such that the set $\{\theta \in \Theta: \MMD(\bQ, \bP_\theta) \leq \inf_{\theta \in \Theta} \MMD( \bQ, \bP_\theta) + c\}$ is bounded.
\end{assumption}
\begin{assumption} \label{ass:bdd-kernel}
    The kernel $k$ is such that $| k(\xi, \xi') | \leq M$, $M < \infty$, for any $\xi, \xi' \in \Xi$.
\end{assumption}
Assumption \ref{ass:minimiser} ensures that a minimiser in (\ref{eq:npl-target-mmd}) exists and is a common assumption made in MMD estimator methods \citep[see][]{briol2019statistical}. Assumption \ref{ass:bdd-kernel} is needed to obtain a concentration inequality for the NPL posterior and it is often made in methods using MMD estimators \citep[see e.g.][]{briol2019statistical, cherief2022finite, dellaporta2022robust, alquier2024universal}. Many commonly used kernels are bounded, such as the Gaussian kernel.
\begin{theorem} \label{thm:gen-bound}
    Suppose Assumptions \ref{ass:minimiser}-\ref{ass:bdd-kernel} hold. Then with probability at least $1 - \delta$:
    \begin{align} \label{eq:concentration-bound}
    \begin{split}
        \MMD(\bP^\star,\prednplapprox) \leq \inf_{\theta \in \Theta} \MMD(\bP_{\theta}, \bP^\star)
        + C_{n,M,\alpha}
    \end{split}
    \end{align}
    where $C_{n,M,\alpha}$ is a constant depending on the number of samples $n$, the upper bound of the kernel $M$ and the concentration parameter on the DP prior $\alpha$. 
\end{theorem}

\begin{remark}
   $C_{n,M,\alpha}$ has an overall rate of $1/\sqrt{n}$ consistent with existing results for minimum MMD estimators \citep{briol2019statistical, cherief2022finite}. Moreover, given an upper bound of the kernel $M$, the constant is fully known. Hence, if $\inf_{\theta \in \Theta} \MMD(\bP_\theta, \bP^\star)$ can be reasonably approximated, this result can be used to select the radius ensuring \robas includes $\bP^\star$ with high probability. 
\end{remark}
We can now obtain an upper bound for the target optimisation problem for large enough $\epsilon$. 
\begin{corollary} \label{cor:upper-bd}
    Suppose Assumptions \ref{ass:minimiser}-\ref{ass:bdd-kernel} hold and let $C_{n,M,\alpha}$ as in \Cref{thm:gen-bound}. Then, for $\epsilon \geq C_{n,M,\alpha} + \inf_{\theta \in \Theta} \MMD(\bP_{\theta}, \bP^\star)$, with probability at least $1 - \delta$:
    \begin{align*}
        &\bE_{\xi \sim \bP^\star}[f_x(\xi)]  \leq \sup_{\cB_\epsilon^k(\prednplapprox)}~{\bE_{\xi\sim\bP}~\left[ f_x(\xi) \right]}.
    \end{align*}
\end{corollary}
In the special case of Huber's contamination model \citet{huber1992robust}, we can obtain a guarantee similar to \thm \ref{thm:gen-bound} which depends on the contamination level. 

\begin{corollary}[Huber's cont. model] \label{cor:hub-upper-bd}
     Suppose $\bP^\star = (1 - \eta) \bP_{\theta_0} + \eta \bQ$ for some $\theta_0 \in \Theta$, $\bQ \in \cP(\Xi)$ and $\eta \in [0,1]$. Suppose Assumptions \ref{ass:minimiser}-\ref{ass:bdd-kernel} hold and let $C_{n,M,\alpha}$ as in \Cref{thm:gen-bound}. Then with probability at least $1 - \delta$:
$ 
        \MMD(\bP_{\theta_0}, \prednplapprox)   \leq
        4 \eta + 2 C_{n,M,\alpha}.
$
\end{corollary}

\section{Experiments} \label{sec:experiments}
We evaluate our method on several different DGPs, model families and misspecification settings for two decision-making problems: the Newsvendor and the Portfolio. 

We compare our method to existing Bayesian formulations of DRO —\drobas \citep{dellaporta2024decision} and Bayesian DRO (\bdro) \citep{shapiro2023bayesian}—both based on the KL divergence and standard Bayesian posterior. To assess how much robustness in our framework is gained through the robust posterior compared to the choice of the MMD in the ambiguity set, we further compare against the empirical method which uses an MMD ball around the empirical measure (denoted by Empirical MMD). This was presented in \citet{staib2019distributionally} and also forms a special case of Kernel DRO \citep{zhu2021kernel}. Implementation details are provided in Appendix \ref{app:exps}.

We explore two types of misspecification.
\emph{1. Model Misspecification} which occurs when the DGP $\bP^\star$ does not belong to the model family $\cP_\Theta$, e.g. if $\bP^\star$ is multimodal while $\cP_\Theta$ assumes unimodality. This affects Bayesian DRO methods (\bdro, \drobas, \drorobas) but not empirical approaches, as the latter do not rely on a model.
\emph{2. Huber's contamination model} \citep{huber1992robust} which is a specific type of model misspecification (see Fig. \ref{fig:outliers}) wherein the training DGP is $\bP^\star = (1-\eta) \tilde{\bP} + \eta \bQ$ for some $\eta \in [0,1]$ and $\tilde{\bP}, \bQ \in \cP(\Xi)$. Contamination, limited to the training set, impacts both Bayesian and empirical DRO methods since the test distribution is assumed to be $\tilde{\bP}$. Huber contamination relates to concepts like distribution shift and out-of-distribution robustness \citep[e.g.,][]{liu2021towards}.
\subsection{The Newsvendor Problem}

\begin{figure}[t]
\vskip 0.2in
\begin{center}
\centerline{\includegraphics[width=0.9\linewidth]{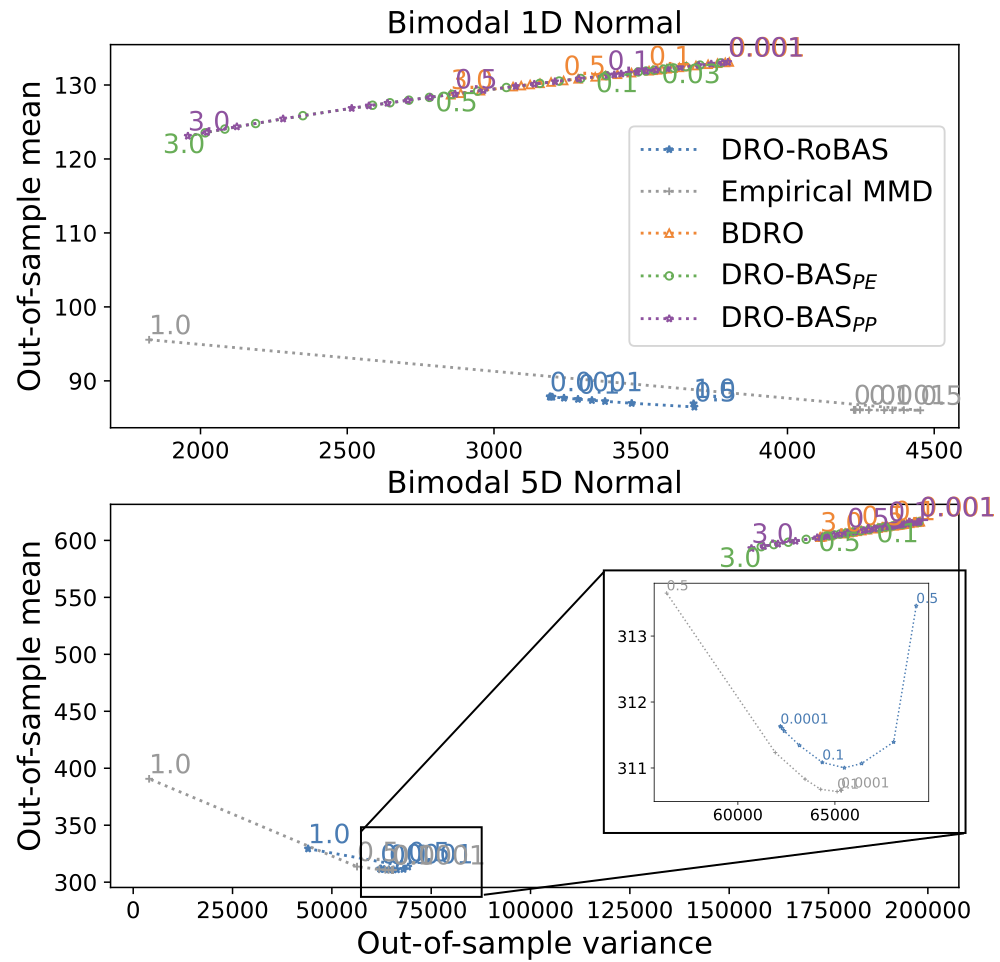}}
\caption{The out-of-sample mean and variance for the Newsvendor problem with a \emph{misspecified} Gaussian location model and a bimodal Gaussian DGP. Results are shown for the univariate ($D=1$, top) and the multivariate ($D=5$, bottom) cases, with markers representing $\epsilon$ values. For illustration purposes, the bottom-left area of the multivariate case is shown in a zoomed-in view.}
\label{fig:bimodal-robas}
\end{center}
\vskip -0.2in
\end{figure}

\begin{figure*}[t]
\vskip 0.2in
\begin{center}
\centerline{\includegraphics[width=0.9\linewidth]{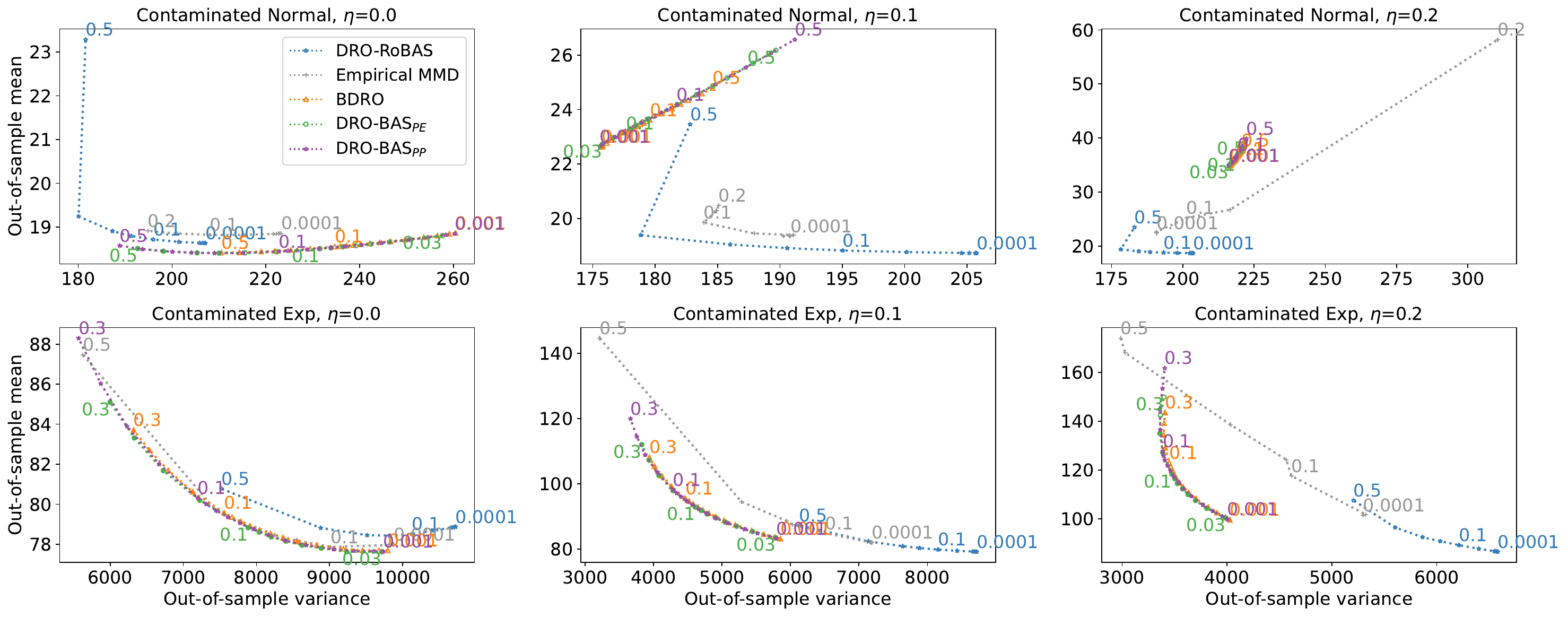}}
\caption{Out-of-sample mean-variance trade-off in the Newsvendor problem for a Gaussian location model (top) and an Exponential model (bottom) with a contaminated training dataset. Results are shown for contamination levels $\eta = 0.0$ (left), $\eta = 0.1$ (middle), and $\eta = 0.2$ (right). Each marker represents a specific $\epsilon$ value, with some labelled for reference.}
\label{fig:cont-normal-robas}
\end{center}
\vskip -0.2in
\end{figure*}

We start with the commonly explored Newsvendor problem \citep[e.g.][]{porteus1990stochastic}. The goal is to choose the optimal amount of products to buy based on consumers' demand. The cost is defined as:
$
f(x,\xi) := h \max(x - \xi, 0) + b \max(\xi-x, 0) 
$
where $x \in \bR^D_{\geq 0}$ is the number of product units ordered, $\xi \in \bR^D$ is the consumers' demand, $0 \in \bR^D$ is the zero vector and $b$ and $h$ denote the backorder and holding cost per unit respectively. In all examples we follow the implementation of \citet{shapiro2023bayesian, dellaporta2024decision} and set $b = 8$ and $h = 3$. We run each experiment $J=100$ times for $n = 20$ observations and compute the out-of-sample mean and variance of the cost incurred.

We consider two models and several DGP cases as follows. 
    First, we assume the demand $\xi \in \Xi$ follows a Gaussian distribution with known variance, i.e. $\bP_\theta := \cN(\theta, \sigma^2 \cI_{D \times D})$ while the DGP is a bimodal Gaussian distribution (case 1 above): 
    $\truedist:= 0.5 \cN(\theta^\star_1, \sigma^2 \cI_{D \times D}) + 0.5 \cN(\theta^\star_2, \sigma^2 \cI_{D \times D})$ for $D=1,5$. 
    Furthermore, we explore the same Gaussian model 
    with a contaminated Gaussian training DGP (case 2 above):
    $\truedist_{\text{train}}:= (1 -\eta) \cN(\theta^\star, \sigma^2 ) + \eta \cN(\theta^\prime, \sigma^2)$ 
    and an Exponential model $\bP_\theta := \text{Exp}(\theta)$ with a contaminated Exponential DGP for the training DGP (case 2 above): 
     $\bP^\star_{\text{train}} := (1 - \eta) \text{Exp}(\theta^\star) + \eta \cN(\mu, \sigma)$ 
    for $\eta \in \{0.0, 0.1, 0.2\}$.

\fig \ref{fig:bimodal-robas} presents the out-of-sample mean and variance of the methods for the bimodal univariate and multivariate Gaussian DGPs. 
The effect of model misspecification is notably more pronounced for the \drobas and \bdro instantiations, which are based on the standard Bayesian posterior and the KL divergence. 
In the \drorobas case, this robustness is likely due to the fact that the obtained NPL-MMD posterior is bimodal, even though the model itself is unimodal. 
In contrast, the standard Bayesian posterior is highly sensitive to misspecification, resulting in a unimodal posterior concentrated between the two modes. Consequently, \drobas and \bdro require very large values of $\epsilon$ to capture the true DGP, leading to conservative decisions that incur high out-of-sample costs. This is evident as they achieve lower mean-variance as $\epsilon$ increases.

We further observe (\fig \ref{fig:bimodal-robas}) that in the univariate case, empirical MMD achieves a lower (difference of $\leq 1$) out-of-sample mean for most values of $\epsilon < 1$.
However, \drorobas consistently shows a lower out-of-sample variance. In the multivariate case, the performances of the two methods are similar, though empirical MMD outperforms \drorobas in both mean and variance. Notably, this comparison pits a Bayesian method under model misspecification against a completely empirical method unaffected by this misspecification.
However, it is promising that \drorobas remains highly competitive against this baseline. The next example, based on contamination models, illustrates a scenario where robustness is crucial for both model-based and empirical methods.

\begin{figure*}[!ht]
    \centering
    \includegraphics[width=0.9\linewidth]{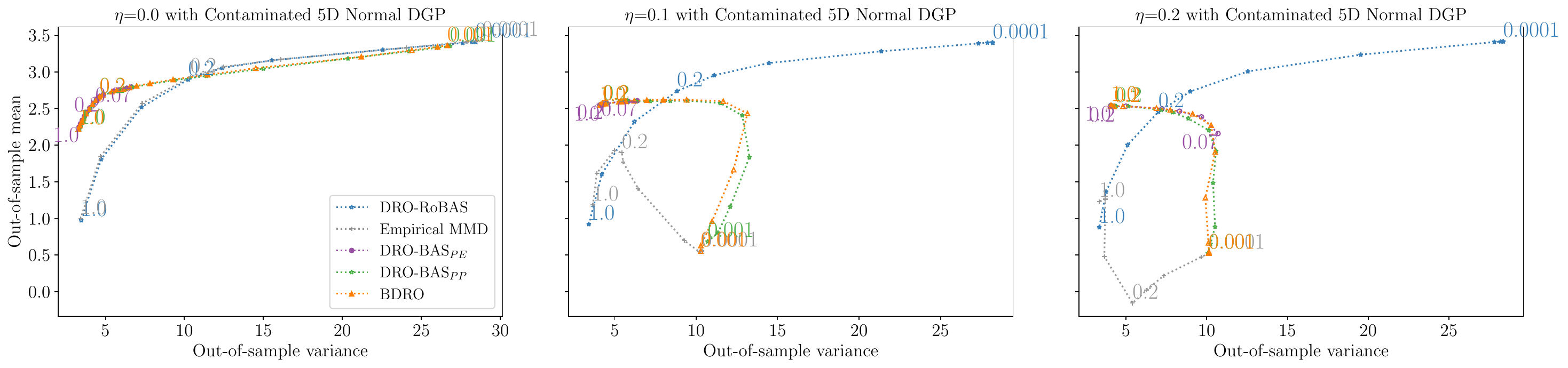}
    \caption{Out-of-sample mean-variance trade-off in the Portfolio problem for a 5D contaminated Gaussian DGP 
    with~$\eta=0.0,0.1,0.2$. Note that the goal is to maximise returns, so {\em larger} out-of-sample mean is better.}
    \label{fig:synthetic-portfolio}
\end{figure*}

In the second simulation, we consider the Huber contamination models \citep{huber1992robust} where the training set is contaminated whereas the test set is not. \fig \ref{fig:cont-normal-robas} demonstrates that both \drobas formulations outperform \drorobas and the empirical MMD method in the well-specified case, where there is no contamination in the training set, and the training and test distributions are identical. In misspecified cases, where $\eta > 0$, \drorobas shows greater robustness compared to the other methods in terms of the out-of-sample mean-variance trade-off. This suggests that the robust posterior and robust distance measures in \robas contribute to a better-informed ambiguity set concerning the \textit{test set generating process}. 
This example further illustrates that, while the motivation for a robust ambiguity set stemmed from concerns about \emph{model misspecification}, even entirely empirical methods, like the empirical MMD, can be sensitive to misspecifications arising from discrepancies between the training and test distributions.

\subsection{The Portfolio Optimisation Problem}

We continue with the multi-dimensional Portfolio problem, also considered by \citet{shapiro2023bayesian}, which chooses stock weightings ($x \in \bR^D$) to maximise returns. The objective function is $f_x(\xi) = - \xi^\top x$ which corresponds to maximising the return and the optimisation is subject to the constraints $x_i \geq 0$ for all $i = 1,\ldots, D$ and $\sum_{i=1}^D x_i = 1$. We generate $n=100$ observations from a 5$D$ Gaussian DGP with contamination on three-out-of-five dimensions: $\bP^\star = (1 - \eta) \cN(\mu^\star, \Sigma^\star) + \eta \cN(\mu', \Sigma^\star)$.
We use a multivariate Gaussian model with unknown mean and variance.

Fig.~\ref{fig:synthetic-portfolio} shows \drorobas is unaffected by the contamination, whilst empirical MMD, \drobas, and \bdro are negatively affected.
Consider~$\eta=0.1$: for $\epsilon < 0.2$, empirical DRO performance quickly degrades compared to \drorobas; but, for $\epsilon \geq 0.2$, empirical MMD performs similarly  due to the MMD robustness.
This effect is magnified for~$\eta=0.2$, demonstrating that the empirical nominal distribution is unreliable due to outliers, whilst \drorobas benefits from a robust nominal - the NPL-MMD posterior predictive - thus performs better for small~$\epsilon$. 

\subsection{Computational Time} \label{sec:comp-time}
The increased robustness of \drorobas comes at the cost of higher computational demands (see 
Appendix \ref{app:exps}). This cost arises from the complex optimisation problem in the RKHS and the longer sampling time required for the NPL posterior. However, as previously demonstrated, this cost is justified by significantly improved out-of-sample performance across various cases of model misspecification. Moreover, by leveraging the NPL-MMD and the MMD, \drorobas can be used for any choice of model family, even likelihood-free models. Notice that \drobaspe is limited to exponential family models whereas the computational cost of \bdro and \drobaspp increases considerably if the posterior is not available in closed form, as methods like Markov Chain Monte Carlo are needed for posterior sampling. This highlights the flexibility and robustness of \drorobas despite its computational demands.

\section{Conclusion} \label{sec:discussion}

Bayesian formulations of DRO for decision-making problems can suffer from model misspecification as the ambiguity set heavily relies on the non-robust Bayesian posterior. We addressed this challenge by using a robust NPL posterior to inform the ambiguity set and leveraging the MMD to construct \textit{both} the posterior and the ambiguity set itself. We show that \drorobas admits a dual formulation in the RKHS and we provide probabilistic guarantees for the tolerance level such that the resulting optimisation problem upper bounds the true objective with high probability. 

Scalability improvements for \drorobas can be achieved through existing tools from kernel methods such as Fourier features \citep{rahimi2007random} and low-rank kernel matrix approximations \citep{bach2013sharp} that are left for future work. Our empirical evidence suggests that if the model is well-specified, or the level of misspecification is low, then existing Bayesian formulations like \drobas can have better performance and scalability. At the same time, when model misspecification is moderate or high then \drorobas achieves significantly improved out-of-sample performance and robustness. Note that any prior knowledge on the misspecification level can be naturally incorporated into our framework to further boost performance.

Finally, the construction of robust Bayesian Ambiguity Sets can extend beyond the choices of NPL and MMD. For example, another instantiation of our framework arises if we employ Generalised Bayesian Inference (GBI) \citep{bissiri2016general}. One of the motivations of GBI is to induce robustness with respect to model misspecification \citep{ghosh2016robust, jewson2018principles, cherief2020mmd,  knoblauch2022optimization, matsubara2022robust, altamirano2023robusta} by targeting a different divergence than the KL. While these methods do not directly impose uncertainty on the DGP like NPL does, they can produce robust GBI posteriors, making it worthwhile to integrate into \drorobas. Notably, the duality results in \Cref{sec:duality} are based on the MMD choice, however, they hold under a general posterior and are \textit{not} dependent on the NPL framework.

\section*{Acknowledgements}
CD acknowledges support from EPSRC grant
[EP/T51794X/1] as part of the Warwick CDT in
Mathematics and Statistics and EPSRC grant [EP/Y022300/1]. PO and TD acknowledge support from a UKRI Turing AI acceleration Fellowship [EP/V02678X/1] and a Turing Impact Award
from the Alan Turing Institute. For the purpose of open access, the authors have applied a Creative Commons Attribution (CC-BY) license to any Author Accepted Manuscript version arising from this submission.

\bibliography{refs}
\bibliographystyle{icml2025}

\newpage
\appendix
\onecolumn
\begin{center}
    {\LARGE \textbf{Supplementary Material}}
\end{center} 
The Supplementary Material is structured as follows. Section \ref{app:proofs} contains the proofs of the theoretical results discussed in \Cref{sec:dro-bas-misspec}, while Section \ref{app:exps} provides further experimental details. Finally, Section \ref{sec:alt} discusses alternative formulations of \drorobas.

\section{Proofs of Theoretical Results}
\label{app:proofs}
In this section, we provide detailed proofs of the theoretical results concerning the duality of the \drorobas problem and the tolerance level guarantees.

\subsection{Proof of \texorpdfstring{\Cref{prop:support-fn}}{proposition}}
First, we derive the support function of the effective ambiguity set in terms of kernel mean embeddings defined in Sec. \ref{sec:duality}: 
\begin{equation} \label{eq:app-robas:c-star}
\begin{split}
\cC^\star &:= 
\{\mu \in \cH_k:  \| \mu - \mu_{\prednpl} \|_{k} \leq \epsilon\} 
\end{split}
\end{equation}
where we denote by $\bE_{\dpxi}[\cdot]$ the expectation under $\bE_{\bQ \sim \DP(c', \bF')}[\cdot]$ and we have defined the NPL-MMD posterior predictive as $\prednpl := \bE_{\DP_{\xi_{1:n}}}[\bP_{\theta_k(\bQ)}]$.  We start with the following \lem. 
\begin{lemma} \label{lem:sup-cauchy}
    For any $\epsilon > 0$ we have:
    \begin{align*}
    \sup_{\| \mu - \mu_{\prednpl} \|_{k} \leq \epsilon} \left< g, \mu - \mu_{\prednpl}\right>_{\cH_k} = \epsilon \| g \|_{k}.
    \end{align*}
\end{lemma}
\begin{proof}
    The proof follows the same logic as the proof for the ambiguity set corresponding to the MMD-ball around the empirical measure, provided in Appendix A.2 of \citet{zhu2021kernel}. To prove the equality statement we will prove both sides of the inequality. 
    We first prove that the left-hand side is less than or equal to the right-hand side. Applying the Cauchy-Swarz inequality we obtain:
    \begin{equation}
    \begin{split}
        \sup_{\| \mu - \mu_{\prednpl} \|_k \leq \epsilon} \left< g, \mu - \mu_{\prednpl}\right>_{\cH_k} 
        \leq \sup_{\| \mu - \mu_{\prednpl} \|_k \leq \epsilon}  \|g\|_k \| \mu - \mu_{\prednpl}\|_k 
         = \epsilon \| g \|_k.
    \end{split}
    \end{equation}
    For the opposite direction, let $\mu' := \mu_{\prednpl} + \epsilon \frac{g}{\|g\|_k}$. Then 
    \begin{align*}
        &\| \mu'- \mu_{\prednpl} \|^2_{\cH_k} \\
        &\quad = \left< \mu'- \mu_{\prednpl}, \mu'- \mu_{\prednpl}  \right>_{\cH_k} \\
        & \quad = \left< \mu', \mu' \right>_{\cH_k}  - 2 \left< \mu', \mu_{\prednpl} \right>_{\cH_k} +  \left< \mu_{\prednpl}, \mu_{\prednpl}  \right>_{\cH_k} \\
        &  \quad = \left< \mu_{\prednpl} + \epsilon \frac{g}{\|g\|_k}, \mu_{\prednpl} + \epsilon \frac{g}{\|g\|_k}  \right>_{\cH_k} +
       \left< \mu_{\prednpl}, \mu_{\prednpl}  \right>_{\cH_k} \\
        & \quad \quad - 2 \left< \mu_{\prednpl} + \epsilon \frac{g}{\|g\|_k}, \mu_{\prednpl} \right>_{\cH_k} \\
        &  \quad= \left<\epsilon \frac{g}{\|g\|_k}, \epsilon \frac{g}{\|g\|_k} \right>_{\cH_k} \\
        &  \quad = \epsilon^2 \frac{\left< g, g \right>_{\cH_k}}{\|g\|_k^2} \\
        &  \quad = \epsilon^2
    \end{align*}
    and hence
    \begin{align*}
        \| \mu^\prime - \mu_{\prednpl} \|_k = \epsilon
    \end{align*}
    which proves that $\mu' \in \cC^\star$. To complete the proof it suffices to show that 
    \begin{align}
        \left<g, \mu' - \mu_{\prednpl} \right> = \epsilon \|g\|_k.
    \end{align}
    By definition of $\mu'$ we have:
    \begin{align*}
        \left<g, \mu' - \mu_{\prednpl} \right>_{\cH_k} &= \left<g, \mu_{\prednpl} + \epsilon \frac{g}{\|g\|_k} - \mu_{\prednpl} \right>\\
        &= \left<g, \epsilon \frac{g}{\|g\|_k} \right> \\
        &= \epsilon \|g\|_k.
    \end{align*}
Since $\mu' \in \cC^\star$ we have 
\begin{align*}
    \sup_{\| \mu - \mu_{\prednpl} \|_k] \leq \epsilon} \left< g, \mu - \mu_{\prednpl}\right>_{\cH_k} \geq \left<g, \mu' - \mu_{\prednpl} \right>_{\cH_k}  = \epsilon \|g\|_k
\end{align*}
which completes the proof. 
\end{proof}

We are now ready to prove \Cref{prop:support-fn}.
\begin{proof}[Proof of Proposition \ref{prop:support-fn}]
    The result follows from the definition of the support function by applying \lem \ref{lem:sup-cauchy} and the reproducing property as follows: 
    \begin{align*}
        \delta^\star_{\cC^\star}(g) &= \sup_{\mu \in \cC^\star} \left< g, \mu\right>_{\cH_k} \\
        &= \left<g, \mu_{\prednpl} \right>_{\cH_k} + \sup_{\mu \in \cC^\star} \left< g, \mu\right>_{\cH_k} - \left<g, \mu_{\prednpl} \right>_{\cH_k} \\
        &= \left<g, \mu_{\prednpl} \right>_{\cH_k}  + \sup_{\mu \in \cC^\star} \left< g, \mu - \mu_{\prednpl}\right>_{\cH_k} \\
        &\overset{(1)}{=}  \bE_{\dpxi} \left[ \bE_{\bP_{\thetamapk(\bQ)}} \left[ g(x) \right] \right] + 
         \sup_{\| \mu - \mu_{\prednpl} \|_k \leq \epsilon} \left< g, \mu - \mu_{\prednpl}\right>_{\cH_k} \\
        &\overset{(2)}{=} \bE_{\dpxi} \left[ \bE_{\bP_{\thetamapk(\bQ)}} \left[ g(x) \right] \right]+ \epsilon \| g \|_k.
    \end{align*}
    Equality (1) follows from the reproducing property and the definition of a kernel mean embedding by noting that: 
\begin{align}
    \mu_{\prednpl} := \bE_{\xi \sim \prednpl}[\phi(\xi)] = \bE_{Q \sim \dpxi}[\bE_{\xi \sim \bP_{\thetamapk(\bQ)}}[\phi(\xi)]] = \bE_{\dpxi}[\mu_{\bP_{\thetamapk(\bQ)}}]
\end{align}
where $\phi$ denotes the feature map associated with kernel $k$. Equality (2) follows from \lem \ref{lem:sup-cauchy}.
\end{proof}
\subsection{Proof of Corollary \ref{cor:bas-dro-mmd-dual}}
To prove the Corollary it suffices to verify the assumptions of Theorem \ref{thm:kernel-dro} for our set $\cC^\star$. 

\begin{lemma}
   $\cC^\star$, as defined in (\ref{eq:app-robas:c-star}), is closed and convex. Furthermore, ambiguity set $\cK_{\cC^\star}$ satisfies $\ri(\cK_{\cC^\star}) \neq \emptyset$, where
   \begin{align} \label{eq:app:Kc}
\begin{split}
    \cK_{\cC^\star} := \{\bP \in \cP_k(\Xi), \mu_\bP \in \cC^\star\} = \{ \bP: \MMD(\bP, \prednpl) \leq \epsilon \}.
\end{split}
\end{align}
\end{lemma}
\begin{proof}
    \textit{Convexity:} Convexity follows trivially by the triangle inequality. Let $\mu_1, \mu_2 \in \cC^\star$ and $\lambda \in [0,1]$. Then 
    \begin{align*}
        \left\| \lambda \mu_1 + (1-\lambda) \mu_2 - \mu_{\prednpl} \right\|_k 
        & \leq  \lambda \left\| \mu_1 - \mu_{\prednpl} \right\|_k + (1-\lambda) \left\| \mu_2 - \mu_{\prednpl} \right\|_k \\
        & \leq \lambda \epsilon + (1-\lambda) \epsilon = \epsilon
    \end{align*}
    where the first inequality follows from the triangle inequality applied to the RKHS norm. Hence, $\lambda \mu_1 + (1- \lambda) \mu_2 \in \cC^\star$ and $\cC^\star$ is convex. 

    \textit{Closeness of $\cC^\star$:} By definition, a closed set is a set whose complement is open, so it suffices to show the set~$\cH_k \backslash \cC^\star $ is open.
    Again by definition, set~$\cH_k \backslash \cC^\star $ is open if, for all $y \in \cH_k \backslash \cC^\star $, there exists $r > 0$ such that any point~$z \in \cH_k$ satisfying~$\Vert y - z \Vert_{\cH_k} < r$ also belongs to $\cH_k \backslash \cC^\star $.
    Given any~$y \in \cH_k \backslash \cC^\star $, let~$r =  \Vert y - \mu_{\prednpl} \Vert_{\cH_k}  - \epsilon$.
    Observe that $r > 0$ because $\Vert y - \mu_{\prednpl} \Vert_{\cH_k}  > \epsilon$.
    Now, for all~$z \in \cH_k$ such that~$\Vert y - z \Vert_{\cH_k} < r$, we can apply the triangle inequality to obtain
    $$\Vert y - \mu_{\prednpl} \Vert_{\cH_k}  \leq  \Vert z - \mu_{\prednpl} \Vert_{\cH_k} + \Vert z - y \Vert_{\cH_k} .$$
    After rearranging, we have
    \begin{align*}
        \Vert z - \mu_{\prednpl} \Vert_{\cH_k}  &\geq \Vert y - \mu_{\prednpl} \Vert_{\cH_k} - \Vert z - y \Vert_{\cH_k} \\
        &> \Vert y - \mu_{\prednpl} \Vert_{\cH_k}  - r = \epsilon.
    \end{align*}
    Therefore, $ \Vert z - \mu_{\prednpl} \Vert_{\cH_k}  > \epsilon$, so $z$ lies in set $\cH_k \backslash \cC^\star $, which concludes our claim that $\cH_k \backslash \cC^\star $ is open and proves that $\cC^\star$ is closed. 

    \textit{Non-empty relative interior of $\cK_{\cC^\star}$:} It suffices to prove that $\cK_{\cC^\star}$ is non-empty and convex. Convexity follows from an analogous proof of convexity of $\cC^\star$ since MMD satisfies the triangle inequality.  $\cK_{\cC^\star}$ is non-empty since, for any $\epsilon > 0$, $\prednpl \in \cK_{\cC^\star}$. 
\end{proof}

\subsection{Proof of Theorem \ref{thm:gen-bound}.}
We first provide a Lemma adapted from the results of \citet{dellaporta2022robust} which bounds the expected MMD (under a finite DGP sample) between the DGP and a sample from the approximate DP posterior. This is necessary to quantify the generalisation error between the DGP and the obtained model $\bP_{\thetamapk(\bQ)}$ for a DP sample $\bQ$. 
\begin{lemma} \label{lemma:from-npl-mmd}
     Let $\bQ \sim \dphatxi$ be defined as in (\ref{eq:approx-dp}). Assume that $\xi_{1:n}\simiid \bP^\star$ and that the kernel $k$ is such that $| k(\xi,\xi') | \leq K$, $K < \infty$, for any $\xi, \xi' \in \Xi$. Then
    \begin{equation}
        \bE_{\xi_{1:n}\simiid \bP^\star} \left[\bE_{\dphatxi} \left[\MMD(\bP^\star, \bQ)\right]\right] 
        \leq \sqrt{\frac{K}{n}} + \sqrt{\frac{2K(n-1) + c(c + 1)}{(c + n)(c + n + 1)}} + \sqrt{\frac{Kc(c + 1)}{(c + n)(c + n + 1)}}.
    \end{equation}
\end{lemma}
\begin{proof}
    The result follows directly from Lemmas 6 and 11 of \citet{dellaporta2022robust} with some small modifications. First notice that by the triangle inequality 
    \begin{align} \label{eq:triangle-pstar-q}
        \MMD(\bP^\star, \bQ) \leq \MMD(\bP^\star, \hat{\bP}_n) + \MMD(\hat{\bP}_n, \bQ)
    \end{align}
    where $\hat{\bP}_n := \frac{1}{n} \sum_{i=1}^n \delta_{\xi_{1:n}}$. By \citet{dellaporta2022robust}, \lem 6, we have that 
    \begin{align} \label{eq:bound-pstar-pn}
        \bE_{\xi_{1:n} \simiid \bP^\star}\left[ \MMD (\bP^\star, \hat{\bP}_n) \right] \leq \sqrt{\frac{K}{n}}. 
    \end{align}
    Moreover, by \citet{dellaporta2022robust}, \lem 11, we have that:
    \begin{align} \label{eq:bound-pn-q}
        \bE_{\xi_{1:n} \simiid \bP^\star} \left[\bE_{\bQ \sim \dphatxi} \left[\MMD(\hat{\bP}_n, \bQ) \right] \right] &\leq \sqrt{\frac{2K(n-1) + c(c + 1)}{(c + n)(c + n + 1)}} + \sqrt{\frac{Kc(c + 1)}{(c + n)(c + n + 1)}}.
    \end{align}
    Taking expectations and using equations (\ref{eq:bound-pstar-pn}) and (\ref{eq:bound-pn-q}) in (\ref{eq:triangle-pstar-q}) we obtain the required result:
    \begin{equation} \begin{split} 
        \bE_{\xi_{1:n} \simiid \bP^\star} \left[\bE_{\bQ \sim \dphatxi} \left[\MMD(\bP^\star, \bQ)\right]\right] 
        & \leq \bE_{\xi_{1:n} \simiid \bP^\star} \left[\MMD(\bP^\star, \hat{\bP}_n)\right] + \bE_{\xi_{1:n} \simiid \bP^\star} \left[\bE_{\bQ \sim \dphatxi} \left[\MMD(\hat{\bP}_n, \bQ) \right] \right] \\
        & \leq \sqrt{\frac{K}{n}} + \sqrt{\frac{2K(n-1) + c(c + 1)}{(c + n)(c + n + 1)}} + \sqrt{\frac{Kc(c + 1)}{(c + n)(c + n + 1)}}.
    \end{split} \end{equation}
\end{proof}
Based on this bound, we can now provide a result in probability. 
\begin{lemma} \label{lemma:conc-pstar-q}
Assume that $\xi_{1:n} \simiid \bP^\star$ and that kernel $k$ is such that $| k(\xi, \xi') | \leq K$, $K < \infty$, for any $\xi, \xi' \in \Xi$. Then with probability $1 - \delta$ we have:
\begin{equation} \begin{split}  
    \bE_{\bQ \sim \dphatxi}[\MMD(\bP^\star, \bQ)] &\leq
        \sqrt{\frac{K}{n}} + \sqrt{\frac{2K(n-1) + c(c + 1)}{(c + n)(c + n + 1)}} \\
        & \quad + \sqrt{\frac{Kc(c + 1)}{(c + n)(c + n + 1)}} + \frac{\sqrt{2nK \left(\log{\frac{1}{\delta}} \right)}}{c + n}.
\end{split} \end{equation}
\end{lemma}
\begin{proof}
    We follow the proof technique of Lemma 1 of \citet{briol2019statistical} based on McDiarmid's inequality \citep{mcdiarmid1989method}. First, notice that by definition of $\dphatxi$ in (\ref{eq:approx-dp}) we can re-write the objective as
    \begin{equation} \begin{split} 
        \bE_{\bQ \sim \dphatxi}[\MMD(\bP^\star, \bQ)] = \bE_{w \sim \Dir} \left[\bE_{\tilde{\xi}_{1:T} \simiid \bF}\left[\MMD(\bP^\star, \widehat{\bQ}_{\xi_{1:n}})\right]\right]
    \end{split} \end{equation}
    where $\widehat{\bQ}_{\xi_{1:n}} = \sum_{i=1}^n w_i \delta_{\xi_i} + \sum_{k=1}^\tau \tilde{w}_k \delta_{\tilde{\xi}_k} \sim \dphatxi$ and $\bE_{w \sim \Dir}$ denotes the expectation under $\Dir\left(1, \dots 1, \frac{c}{\tau}, \dots, \frac{c}{\tau}\right)$. 
    Let $h(\xi_1, \dots, \xi_n) := \bE_{w \sim \Dir} \left[\bE_{\tilde{\xi}_{1:T} \simiid \bF}\left[\MMD(\bP^\star, \widehat{\bQ}_{\xi_{1:n}})\right]\right]$. Then for all $\{\xi_i\}_{i=1}^n, \xi_i' \in \Xi$ and $\xi'_{1:n} := \{\xi_1, \dots, \xi_{i-1}, \xi_i', \xi_{i+1}, \dots, \xi_n\}$ we have 
    \begin{equation} \begin{split} 
        &\left | 
            h(\xi_1, \dots, \xi_{i-1}, \xi_i, \xi_{i+1}, \dots, \xi_n) - h(\xi_1, \dots, \xi_{i-1}, \xi_i', \xi_{i+1}, \dots, \xi_n)
        \right |\\
        &\quad \quad = 
         \left | 
            \bE_{w \sim \Dir} \left[\bE_{\tilde{\xi}_{1:T} \simiid \bF}\left[\MMD(\bP^\star, \widehat{\bQ}_{\xi_{1:n}})\right]\right] - \bE_{w \sim \Dir} \left[\bE_{\tilde{\xi}_{1:T} \simiid \bF}\left[\MMD(\bP^\star, \widehat{\bQ}_{\xi'_{1:n}})\right]\right]
        \right | \\
        &\quad \quad =
         \left | 
            \bE_{w \sim \Dir} \left[\bE_{\tilde{\xi}_{1:T} \simiid \bF}\left[\MMD(\bP^\star, \widehat{\bQ}_{\xi_{1:n}}) - \MMD(\bP^\star, \widehat{\bQ}_{\xi'_{1:n}})\right]\right] 
        \right | \\
        & \quad \quad \leq 
        \bE_{w \sim \Dir} \left[\bE_{\tilde{\xi}_{1:T} \simiid \bF}\left[ \left | \MMD(\bP^\star, \hat{\bQ}_{\xi_{1:n}}) - \MMD(\bP^\star, \hat{\bQ}_{\xi'_{1:n}}) \right | \right]\right] \\
        & \quad \quad = 
        \bE_{w \sim \Dir} \left[\bE_{\tilde{\xi}_{1:T} \simiid \bF}\left[ \left | \| \mu_\bP^\star - \mu_{\widehat{\bQ}_{x_{1:n}}} \|_k - \| \mu_\bP^\star - \mu_{\widehat{\bQ}_{\xi'_{1:n}}} \|_k \right | \right]\right] \\
        & \quad \quad \leq 
        \bE_{w \sim \Dir} \left[\bE_{\tilde{\xi}_{1:T} \simiid \bF}\left[ \| \mu_\bP^\star - \mu_{\widehat{\bQ}_{\xi_{1:n}}} - \mu_\bP^\star + \mu_{\widehat{\bQ}_{\xi'_{1:n}}} \|_k \right]\right] \\
        & \quad \quad = 
        \bE_{w \sim \Dir} \left[\bE_{\tilde{\xi}_{1:T} \simiid \bF}\left[  \| w_i (k(\xi_i, \cdot) - k(\xi_i', \cdot)) \|_k \right]\right] \\
        & \quad \quad = \bE_{w \sim \Dir} \left[ w_i \| k(\xi_i, \cdot) - k(\xi_i', \cdot) \|_k \right] \\
        & \quad \quad \leq
        \bE_{w \sim \Dir} [ w_i] 2 \sqrt{K} \\
        & \quad \quad = 
        \frac{2 \sqrt{K}}{c + n} 
    \end{split} \end{equation}
    where in the first inequality we used Jensen's inequality and in the second we used the reverse triangle inequality. Then by McDiarmid's inequality \citep{mcdiarmid1989method} we have that for any $\epsilon > 0$:
    \begin{equation} \begin{split} 
        &\bP \left( \bE_{w \sim \Dir} \left[\bE_{\tilde{\xi}_{1:T} \simiid \bF}\left[\MMD(\bP^\star, \hat{\bQ}_{\xi_{1:n}})\right]\right] \right.\\
        &\left. \quad \quad - \bE_{\xi_{1:n} \simiid \bP^\star} \left[\bE_{w \sim \Dir} \left[\bE_{\tilde{\xi}_{1:T} \simiid \bF}\left[\MMD(\bP^\star, \widehat{\bQ}_{\xi_{1:n}})\right]\right] \right] \geq \epsilon \right) \\
        & \leq \exp{\left(\frac{- 2\epsilon^2}{\frac{4nK}{(n+c)^2}}\right)} \\
        & =
        \exp{\left(\frac{- \epsilon^2 (n+c)^2}{2nK}\right)}
    \end{split} \end{equation}
    Let $\delta :=\exp{\left(\frac{- \epsilon^2 (n+c)^2}{2nK}\right)}$ then it follows that with probability $1 - \delta$:
    \begin{equation} \begin{split} 
        &\bE_{w \sim \Dir} \left[\bE_{\tilde{\xi}_{1:T} \simiid \bF}\left[\MMD(\bP^\star, \hat{\bQ}_{\xi_{1:n}})\right]\right] \\
        & \quad \quad \leq \bE_{\xi_{1:n} \simiid \bP^\star} \left[\bE_{w \sim \Dir} \left[\bE_{\tilde{\xi}_{1:T} \simiid \bF}\left[\MMD(\bP^\star, \hat{\bQ}_{\xi_{1:n}})\right]\right] \right] + \epsilon \\
        &\quad \quad  = \bE_{\xi_{1:n} \simiid \bP^\star} \left[\bE_{w \sim \Dir} \left[\bE_{\tilde{\xi}_{1:T} \simiid \bF}\left[\MMD(\bP^\star, \hat{\bQ}_{\xi_{1:n}})\right]\right] \right] + \frac{\sqrt{2nK \left(\log{\frac{1}{\delta}} \right)}}{n + c}\\
        & \quad \quad \leq \sqrt{\frac{K}{n}} + \sqrt{\frac{2K(n-1) + c(c + 1)}{(c + n)(c + n + 1)}} + \sqrt{\frac{Kc(c + 1)}{(c + n)(c + n + 1)}} \\
        &\quad \quad \quad \quad  + \frac{\sqrt{2nK \left(\log{\frac{1}{\delta}} \right)}}{c + n}
    \end{split} \end{equation}
where the last inequality follows from \lem \ref{lemma:from-npl-mmd}.
\end{proof}
We are now ready to prove the main result. 
\begin{proof}[Proof of Theorem \ref{thm:gen-bound}]
    First note that
    \begin{equation} \begin{split} 
        \MMD(\bP^\star, \bP_{\thetamapk(\bQ)}) - \inf_{\theta \in \Theta} \MMD(\bP_{\theta}, \bP^\star)
        &\leq \MMD(\bP^\star, \bQ) + \MMD(\bQ, \bP_{\thetamapk(\bQ)})  - \inf_{\theta \in \Theta} \MMD(\bP_{\theta}, \bP^\star)\\
        &=\MMD(\bP^\star, \bQ) + \inf_{\theta \in \Theta} \MMD(\bP_{\theta}, \bQ) - \inf_{\theta \in \Theta} \MMD(\bP_{\theta}, \bP^\star) \\
        & \leq \MMD(\bP^\star, \bQ) + \left | 
        \inf_{\theta \in \Theta} \MMD( \bP_{\theta}, \bQ) - \inf_{\theta \in \Theta} \MMD(\bP_{\theta}, \bP^\star) \right | \\
        & \leq \MMD(\bP^\star, \bQ) + \sup_{\theta \in \Theta} \left | 
       \MMD(\bP_{\theta}, \bQ) - \MMD(\bP_{\theta}, \bP^\star) \right | \\
       &\leq \MMD(\bP^\star, \bQ) + \sup_{\theta \in \Theta} \MMD(\bQ, \bP^\star) \\
       &= 2  \MMD(\bP^\star, \bQ) \label{eq:here}
    \end{split} \end{equation}
    where we used the triangle inequality in the first inequality, the definition of $\thetamapk(\cdot)$ in the first equality and the reversed triangle inequality in the last inequality. The third inequality follows from the fact (mentioned in \citet{briol2019statistical}, \thm 1) that since $k$ is bounded, the family $\inf_{\theta \in \Theta} \MMD(\bP_{\theta}, \cdot)$ is uniformly bounded and for bounded functions $f,g$ we have $\|\inf_\theta f(\theta) - \inf_\theta g(\theta) \| \leq \sup_{\theta} | f(\theta) - g(\theta) |$. Hence, by Jensen's inequality, since the MMD is convex, we have 
    \begin{align} \label{eq:ppred}
        \MMD(\truedist, \prednplapprox) = \MMD(\truedist, \bE_{\dphatxi}[\bP_{\thetamapk(\bQ)}]) 
        \leq \bE_{\dphatxi}[\MMD(\truedist, \bP_{\thetamapk(\bQ)})]
    \end{align}
    and therefore
    \begin{equation} \begin{split} 
         \MMD(\truedist, \prednplapprox) - \inf_{\theta \in \Theta} \MMD(\bP_{\theta}, \bP^\star) 
         &\leq \bE_{\bQ \sim \dphatxi}[\MMD(\bP^\star, \bP_{\thetamapk(\bQ)})] - \inf_{\theta \in \Theta} \MMD(\bP_{\theta}, \bP^\star) \\
        & \leq
        2  \bE_{\bQ \sim \dphatxi}[ \MMD(\bP^\star, \bQ)]. 
    \end{split} \end{equation}
    By \lem \ref{lemma:conc-pstar-q} it follows that with probability $1 - \delta$ the advertised result holds. 
\end{proof}

\subsection{Proof of Corollary \ref{cor:hub-upper-bd}.}
Recall that $\bP^\star = (1 - \eta) \bP_{\theta_0} + \eta \bQ$ for some $\theta_0 \in \Theta$, $\bQ \in \cP(\Xi)$ and $\eta \in [0,1]$. Corollary \ref{cor:hub-upper-bd} is an immediate consequence of Lemma 3.3 of \citet{cherief2022finite} which states that for any $\theta \in \Theta$:
\begin{align}  \label{eq:robas:cont-ineq}
    | \MMD(\bP_\theta, \truedist) - \MMD(\bP_\theta, \bP_{\theta_0}) | \leq  2 \eta.
\end{align}
Using this and Theorem \ref{thm:gen-bound} we can prove the advertised result as follows:
\begin{proof}[Proof of Corollary \ref{cor:hub-upper-bd}]
    It follows by (\ref{eq:robas:cont-ineq}) that:
    \begin{align}
    \MMD(\bP_{\thetamapk(\bQ)}, \bP_{\theta_0}) \leq 2 \eta + \MMD(\bP_{\thetamapk(\bQ)}, \truedist).
    \end{align}
    and hence 
    \begin{align} \label{eq:here2}
         \bE_{\dphatxi}[\MMD(\bP_{\thetamapk(\bQ)}, \bP_{\theta_0})] \leq 2 \eta +  \bE_{\dphatxi}[\MMD(\bP_{\thetamapk(\bQ)}, \truedist)].
    \end{align}

    We then have:
    \begin{align*}
    \MMD(\bP_{\theta_0}, \prednplapprox)
    &\leq \bE_{\dphatxi}[\MMD(\bP_{\theta_0}, \bP_{\thetamapk})] \\
    & \leq 2 \bE_{\dphatxi}[\MMD(\bP_{\theta_0}, \bQ)] + \inf_\theta \MMD(\bP_{\theta_0}, \bP_\theta) \\
    &= 2 \bE_{\dphatxi}[\MMD(\bP_{\theta_0}, \bQ)] \\
    & \leq 4 \eta + 2 \bE_{\dphatxi}[\MMD(\bP^\star, \bQ)]
    \end{align*}
    and the result follows from Lemma \ref{lemma:conc-pstar-q}. The first inequality follows from Jensen's inequality as in (\ref{eq:ppred}), the second inequality follows from \eqref{eq:here}, the equality holds from the fact that $\bP_{\theta_0} \in \cP_\Theta$ and the last inequality follows form \eqref{eq:here2}.
\end{proof}

\section{Additional Experimental Details}
\label{app:exps}
We implemented the dual problems for \drorobas (\Cref{cor:bas-dro-mmd-dual}) in Python using CVXPY version 1.5.2 and the MOSEK solver version 10.1.28 following the implementation of Kernel DRO algorithms by \citet{zhu2021kernel}. We used 200 discretisation points for the constraint, as described in \Cref{para:discretis}. For the implementation of the NPL-MMD posterior, we follow the median heuristic suggested by \citet{gretton2012kernel} and the ADAM optimiser \citep{kingma2014adam} with learning rate $h = 0.1$. For the DP approximation we used $\alpha = 0$ and truncation limit $\tau=100$. For all Bayesian methods we use $N = 900$ total Monte Carlo samples to approximate the expectations in the optimisation objectives.

The out-of-sample mean and variance were calculated as:
\begin{equation}
\begin{split}
    m(\epsilon) &= \frac{1}{JT}\sum_{j=1}^J \sum_{t=1}^T f(x^{(j)}_\epsilon, \xi_{n+t}),\\
    v(\epsilon) &= \frac{1}{JT-1}\sum_{j=1}^J \sum_{t=1}^T \left(f(x^{(j)}_\epsilon, \xi_{n+t}) - m(\epsilon)\right)^2,
\end{split}
\end{equation}
where $x^{(j)}_\epsilon$ is the obtained solution on training dataset~$\xi_{1:n}^{(j)}$ and we set $T = 50$ for the number of test observations.

Detailed results on the average solve and sampling times and associated standard deviations are provided in Tables \ref{tab:solve-time-app} and \ref{tab:sample-time-app}.

\paragraph{Data-generating Process Settings}
For the Newsvendor problem, we first considered a Gaussian distribution with known variance, i.e. $\bP_\theta := \cN(\theta, \sigma^2 \cI_{D \times D})$ while the DGP is a bimodal Gaussian distribution:  $\truedist:= 0.5 \cN(\theta^\star_1, \sigma^2 \cI_{D \times D}) + 0.5 \cN(\theta^\star_2, \sigma^2 \cI_{D \times D})$. We considered a univariate ($d=1$) with $(\theta^\star_1, \theta^\star_2, \sigma) = (10, 60, 5)$ and a multivariate case ($d=5$) with $\theta^\star_1 = (10,20,33,22,25)$, $\theta^\star_2 = (60,60,60,60,60)$ and $\sigma^2 = 5$. Furthermore, we explore the same Gaussian model 
    $\bP_\theta:= \cN(\theta, \sigma^2 \cI_{D \times D})$ 
    with a contaminated Gaussian training DGP:
    $\truedist_{\text{train}}:= (1 -\eta) \cN(\theta^\star, \sigma^2 ) + \eta \cN(\theta^\prime, \sigma^2)$ 
    where we set $(\theta^\star, \theta^\prime, \sigma) = (25, 75, 5)$ and
    for $\eta \in \{0.0, 0.1, 0.2\}$ and an Exponential model $\bP_\theta := \text{Exp}(\theta)$ with a contaminated Exponential DGP for the training DGP: 
    $\bP^\star_{\text{train}} := (1 - \eta) \text{Exp}(0.05) + \eta \cN(100, 0.5)$ 
     $\bP^\star_{\text{train}} := (1 - \eta) \text{Exp}(\theta^\star) + \eta \cN(\mu, \sigma)$ 
    for $\eta \in \{0.0, 0.1, 0.2\}$.

\begin{table*}[!ht]
    \centering
    \caption{Average (standard derivation) {\em solve time} in seconds of algorithms across different DGPs}
        \label{tab:solve-time-app}
    \begin{small}
 . \begin{sc}
    \begin{tabular}{lrrrrr}
    \toprule
     DGP & RoBAS & Empirical MMD & KL-BDRO & \drobaspe & \drobaspp \\
    \midrule
1D $\cN$ bimodal & 275.33 (25.98) & 2.17 (0.25) & 3.90 (0.55) & 0.29 (0.03) & 0.30 (0.05) \\
5D $\cN$ bimodal & 262.94 (31.72) & 2.18 (0.36) & 2.63 (0.26) & 1.24(0.19) & 1.38 (0.29) \\
$\cN, \eta= 0.0$ & 286.27 (16.14) & 2.22 (0.28) & 4.01 (0.49) & 0.29 (0.04) & 0.30 (0.04) \\
$\cN, \eta= 0.1$ & 284.62 (17.42) & 2.23 (0.24)& 3.99 (0.42) &  0.29 (0.03) & 0.30 (0.04) \\
$\cN, \eta= 0.2$ & 260.67 (35.90) & 2.20 (0.24) & 3.74 (0.52) & 0.28 (0.03) & 0.28 (0.04)\\
Exp, $\eta=0.0$& 311.74 (37.48) & 5.59 (0.68) & 0.43 (0.06) & 0.40 (0.06) & 0.49 (0.06) \\
Exp, $\eta=0.1$  & 296.65 (46.10) & 5.39 (0.54) & 0.44 (0.06) & 0.42 (0.05) & 0.50 (0.06) \\
Exp, $\eta=0.2$ & 282.55 (53.23) & 5.22 (0.47) & 0.44 (0.06) & 0.42 (0.05) & 0.50 (0.06) \\
    \bottomrule
    \end{tabular}
     \end{sc}
    \end{small}
\end{table*}

\begin{table*}[!ht]
    \centering
      \caption{Average (standard derivation) {\em sample time} in seconds of algorithms across different DGPs.}
       \label{tab:sample-time-app}
    \begin{small}
    \begin{sc}
    \begin{tabular}{lrrrrr}
    \toprule
     DGP & RoBAS & Empirical MMD & KL-BDRO & \drobaspe & \drobaspp \\
    \midrule
1D $\cN$ bimodal & 15.65 (4.76) & 0.0 (0.0) &  0.001 (0.002) & 0.0002 (0.0) & 0.0002 (0.0) \\
5D $\cN$ bimodal & 15.86 (4.84) & 0.0 (0.0) & 0.0104 (0.0004) & 0.0006 (0.0003) & 0.001 (0.0004) \\
$\cN, \eta= 0.0$ & 15.85 (4.67) & 0.0 (0.0) &0.0009 (0.002) & 0.0002 (0.0) & 0.0002 (0.0)\\
$\cN, \eta= 0.1$ & 15.78 (4.59) & 0.0 (0.0) & 0.0009 (0.0017) & 0.0002 (0.0) & 0.0002 (0.0)\\
$\cN, \eta= 0.2$ & 15.68 (4.58) & 0.0 (0.0) & 0.0008 (0.0014) & 0.0002 (0.0) & 0.0002 (0.0) \\
Exp, $\eta=0.0$ & 38.10 (5.62) & 0.0 (0.0) & 0.1 (0.01) & 0.0 (0.0) & 0.0 (0.0) \\
Exp, $\eta=0.1$ & 37.77 (4.78) & 0.0 (0.0) & 0.1 (0.01) & 0.0 (0.0) & 0.0 (0.0) \\
Exp, $\eta=0.2$ & 37.82 (4.85) & 0.0 (0.0) & 0.1 (0.01) & 0.0 (0.0) & 0.0 (0.0) \\
\bottomrule
    \end{tabular}
    \end{sc}
    \end{small}
\end{table*}

\section{Alternative \robas formulations} \label{sec:alt}
We now discuss alternative \robas formulations as the ones provided in \citet{ dellaporta2024decision} through the \drobas framework. In particular, in the \drobas case the authors explored two distinct formulations of ambiguity sets: the \baspp (\ref{eq:dro-bas-pp}), mirroring the definition of $\cB_\epsilon^k(\prednpl)$, and the \baspe (\ref{eq:dro-bas-pe}) which is based on a posterior expectation by considering the expected KL distance to the model family. 

In the case of robust Bayesian ambiguity sets, an interesting connection can be made between $\cB_\epsilon(\prednpl)$ and the ambiguity set based on the expected squared MMD, namely:
\begin{align} \label{eq:robas:exp-sqred-mmd}
    \cA^k_\epsilon(\DP(c^\prime, \bF^\prime)) 
    =  \left\{\bP \in \cP_k(\Xi): \bE_{\bQ \sim \DP(c^\prime, \bF^\prime)}\left[
    \MMD^2(\bP, \bP_{\theta_k(\bQ)}) \right] \right\}.
\end{align}
The squared MMD has previously been used in Generalised Bayesian Inference to define a loss function by \citet{cherief2020mmd}. By using the properties of the MMD, we can prove the following equivalence:
\begin{lemma} \label{prp:robas:equiv}
    Let $\bE_{\dpxi}$ denote the expectation under $\bE_{\bQ \sim \DP(c^\prime, \bF^\prime)}$ and suppose:
    \begin{equation}
    \begin{split} \label{eq:g-npl}
     v(\DP) &:= \bE_{\dpxi}\left[\left< \mu_{\bP_{\theta_k(\bQ)}}, \mu_{\bP_{\theta_k(\bQ)}}\right>\right] - \left<\bE_{\dpxi}[\mu_{\bP_{\theta_k(\bQ)}}], \bE_{\dpxi}[\mu_{\bP_{\theta_k(\bQ)}}] \right> \\
    &= \bE_{\dpxi}\left[\left\|\mu_{\bP_{\theta_k(\bQ)}} \right\|^2_{\cH_k}\right] - \left \| \bE_{\dpxi}[\mu_{\bP_{\theta_k(\bQ)}}] \right\|^2_{\cH_k} \\
    &< \infty.
    \end{split}
    \end{equation}
    Then for any $\bP \in \cP_k(\Xi)$:
    \begin{align} \label{eq:robas:mmd-eq}
        \bE_{\dpxi} \left[\MMD^2(\bP, \bP_{\theta_k(\bQ)}) \right] =  
        \MMD^2(\bP, \prednpl) + v(\DP)
    \end{align}
     and for any $\epsilon \geq 0$:
    \begin{align} \label{eq:robas:inclusion}
    \cB_{\epsilon}^k(\prednpl) \equiv \cA^k_{\epsilon^2 + v(\DP)}(\DP(c^\prime, \bF^\prime)).
    \end{align}
\end{lemma}

\begin{proof}[Proof of \Cref{prp:robas:equiv}]
    The proof is based on the reproducing property of the RKHS. Recall that $\bE_{\dpxi}$ denotes the expectation under $\bE_{\bQ \sim \DP(c^\prime, \bF^\prime)}$. For any $\bP \in \cP_k(\Xi)$ and $\epsilon \geq 0$ we have:
    \begin{equation} \label{eq:expected-mmd}
        \begin{split}
            &\bE_{\dpxi}[\MMD^2(\bP, \bP_{\thetamapk(\bQ)})] \\
            &\quad = \left< \mu_{\bP}, \mu_{\bP} \right>_{\cH_k} - 2 \bE_{\dpxi}\left[
            \left<\mu_{\bP}, \mu_{\bP_{\thetamapk(\bQ)}} \right>_{\cH_k}
            \right] +
            \bE_{\dpxi}\left[
        \left<\mu_{\bP_{\thetamapk(\bQ)}}, \mu_{\bP_{\thetamapk(\bQ)}} \right>_{\cH_k}
            \right] \\
            &\quad = \left< \mu_{\bP}, \mu_{\bP} \right>_{\cH_k} - 2 \bE_{\dpxi}\left[
            \left<\mu_{\bP}, \mu_{\bP_{\thetamapk(\bQ)}} \right>_{\cH_k}
            \right] + \left< \bE_{\dpxi}[\mu_{\bP_{\thetamapk(\bQ)}}] , \bE_{\dpxi}[\mu_{\bP_{\thetamapk(\bQ)}}]\right>_{\cH_k} \\
            &\quad \quad - \left< \bE_{\dpxi}[\mu_{\bP_{\thetamapk(\bQ)}}] , \bE_{\dpxi}[\mu_{\bP_{\thetamapk(\bQ)}}]\right>_{\cH_k} +
            \bE_{\dpxi}\left[
        \left<\mu_{\bP_{\thetamapk(\bQ)}}, \mu_{\bP_{\thetamapk(\bQ)}} \right>_{\cH_k}
            \right] \\
            &\quad = \MMD^2(\bQ, \bE_{\dpxi}[\bP_{\thetamapk(\bQ)}]) + v(\DP). 
        \end{split}
    \end{equation}
        This proves the statement of Equation \ref{eq:robas:mmd-eq}. To prove the statement of Equation \ref{eq:robas:inclusion} recall that 
        \begin{align}
               \cB^k_{\epsilon}(\prednpl) &:= \{\bP \in \cP_k(\Xi): \MMD(\bP, \prednpl) \leq \epsilon\}\\
               &= \{\bP \in \cP_k(\Xi): \MMD(\bP, \bE_{\dpxi}[\bP_{\thetamapk(\bQ)}]) \leq \epsilon\}
        \end{align}
        Using (\ref{eq:expected-mmd}) we have that for $\epsilon \geq 0$ and some $\bP \in \cB_\epsilon^k(\prednpl)$:
        \begin{align}
            & \MMD(\bP, \bE_{\dpxi}[\bP_{\thetamapk(\bQ)}]) \leq \epsilon \\
            \iff &\MMD^2(\bP, \bE_{\dpxi}[\bP_{\thetamapk(\bQ)}]) \leq \epsilon^2 \\
            \iff &\bE_{\dpxi}[\MMD^2(\bP, \bP_{\thetamapk(\bQ)})] \leq \epsilon^2 + v(n) \\
            \iff &\bP \in \cA^k_{\epsilon^2 + v(\DP)}(\DP(c^\prime, \bF^\prime)).
        \end{align}
hence, $\cB_{\epsilon}^k(\prednpl) \equiv \cA^k_{\epsilon^2 + v(\DP)}(\DP(c^\prime, \bF^\prime))$, which completes the proof.
\end{proof}

The constant $v(\DP)$ (\ref{eq:g-npl}) is an effective variance term of the kernel mean embedding $\mu_{\bP_{\theta_k(\bQ)}}$ in the RKHS under the nonparametric posterior $\DP \equiv \DP(c^\prime, \bF^\prime)$. It follows directly from Lemma \ref{prp:robas:equiv} that the primal \drorobas problem in (\ref{eq:dro-robas-pp}) is equivalent to the following optimisation problem:
\begin{align}\label{eq:robas:pp}
    \min_{x \in \cX} \sup_{\displaystyle\color{robas}\bP: \bE_{\dpxi}[\MMD^2(\bP,\bP_{\theta_k(\bQ)})] \leq \epsilon^2 + v(\DP)}~{\bE_{x\sim\bP}~\left[ f(z,x) \right]}
\end{align}
or equivalently,
\begin{align} \label{eq:dro-robas-pp-equiv}
    \min_{x \in \cX} \sup_{\displaystyle\color{robas}\bP \in \cA_{\epsilon^2 + v(\DP)}^k(\prednpl)} \bE_{\xi \sim \bP}[f_x(\xi)].
\end{align}
Since these two problems are equivalent, we refer to both of them as \drorobas.

An alternative formulation could be defined, based on the expected MMD distance i.e. by considering the set
\begin{align}
    \left\{\bP \in \cP_k(\Xi): \bE_{\dpxi} \left[\MMD(\bP, \bP_{\theta_k(\bQ)}) \right] \leq \epsilon \right\}
\end{align}
for tolerance level $\epsilon \geq 0$. This mirrors the \baspe in \citet{dellaporta2024decision} and could lead to a differently shaped ambiguity set as it corresponds to an expected MMD-ball. However, deriving the dual formulation for the corresponding worst-case DRO problem based on this set is challenging as it involves deriving the convex conjugate of the expected MMD distance $(\bE_{\dpxi}[\MMD(\cdot, \bP_{\theta_k(\bQ)})])^*$ or equivalently the support function $\delta^\star_{\cC}(g)$ of the corresponding set $\cC := \{\mu \in \cH_k: \bE_{\bQ \sim \DP(c', \bF')} \left [\| \mu - \mu_{\bP_{\theta_k(\bQ)}} \|_{k} \right] \leq \epsilon \}$. This is not trivial and it has not been tackled in this work, however it poses itself as important future work as it can open the way to ambiguity sets whose shape itself is informed by the posterior rather than just their nominal distribution as in \robas.

\end{document}